\documentclass{svproc}

\usepackage{url}

\usepackage[linesnumbered]{algorithm2e}
\usepackage[colorlinks, allcolors=blue]{hyperref}
\usepackage{amsmath}
\usepackage{amssymb}
\usepackage{xfrac}
\usepackage{cite}
\usepackage{booktabs}
\usepackage[capitalize,nameinlink,noabbrev]{cleveref}
\usepackage{multirow}

\crefname{equation}{}{}
\crefname{theorem}{Thm.}{Thm.}
\crefname{figure}{Fig.}{Fig.}
\crefname{section}{Sec.}{Sec.}
\crefname{table}{Table}{Table}

\renewcommand{\epsilon}{\varepsilon}

\DeclareMathOperator*{\argmin}{arg\,min}
\DeclareRobustCommand{\bigO}{%
  \text{\usefont{OMS}{cmsy}{m}{n}O}%
}

\newcommand{\cT}{\mathcal{T}}
\newcommand{\cC}{\mathcal{C}}
\newcommand{\cP}{\mathcal{P}}
\newcommand{\cN}{\mathcal{N}}

\newcounter{prob}

\begin{document}
\mainmatter

\title{A Quantum Annealing-Based Approach to Extreme Clustering}
\titlerunning{A Quantum Annealing-Based Approach to Extreme Clustering}

\author{Tim Jaschek,\inst{1,2}  Marko Bucyk,\inst{1} and Jaspreet S. Oberoi\inst{1,3} }
\authorrunning{T. Jaschek, M. Bucyk, and J. S. Oberoi}
\tocauthor{T. Jaschek, M. Bucyk, and J. S. Oberoi}

\institute{1QB Information Technologies (1QBit), Vancouver, BC,  Canada 
\and
Dept. of Mathematics, University of British Columbia, Vancouver, BC, Canada
\and
School of Engineering Science, Simon Fraser University, Burnaby, BC, Canada
}

\maketitle              

\begin{abstract}
Clustering, or grouping, dataset elements based on similarity can be used not only to classify a dataset into a few categories, but also to approximate it by a relatively large number of representative elements. In the latter scenario, referred to as extreme clustering, datasets are enormous and the  number of representative clusters is large. We have devised a distributed method that can efficiently solve extreme clustering problems using quantum annealing. We prove that this method yields optimal clustering assignments under a separability assumption, and show that the generated clustering assignments are of comparable quality to those of assignments generated by common clustering algorithms, yet can be obtained a full order of magnitude faster.

\end{abstract}

\let\thefootnote\relax\footnotetext{$^*$ T. Jaschek and J. S. Oberoi have contributed equally to this research. Address correspondence to: tim.jaschek@1qbit.com}

\section{Introduction}

Traditionally, clustering approaches have been developed and customized for tasks where the resultant number of clusters $k$ is not particularly high. In such cases, algorithms such as $k$-means++ \cite{arthur2007k}, BIRCH~\cite{birch}, DBSCAN \cite{ester1996density}, and spectral clustering produce high-quality solutions in a reasonably short amount of time. This is because these traditional algorithms scale  well with respect to the dataset cardinality $n$. However, in most cases, the computational complexity of these algorithms, in terms of the number of clusters, is either exponential or  higher-order polynomial. Another common issue is that some of the algorithms require vast amounts of memory.

The demand for clustering algorithms capable of solving problems with larger values of $k$ is continually increasing. Present-day examples involve deciphering the content of billions of web pages by grouping them into millions of labelled categories~\cite{nayak2014clustering, de2015parallel}, identifying similarities among billions of images using nearest-neighbour detection \cite{wang2013duplicate, liu2007clustering,woodley2018parallel}.  This domain of clustering, where $n$ and $k$ are both substantially large, is referred to as \textit{extreme clustering}~\cite{kobren2017hierarchical}. Although there is great value in perfecting this type of clustering, very little effort towards this end has been made by the machine learning community. Our algorithm is, in fact, such an effort. Its output is a \emph{clustering tree}, which can be used to generated multiple clustering assignments (or ``levels'') with varying degrees of accuracy (i.e., coarseness or fineness) of the approximation. Generating such a tree is not uncommon for clustering algorithms. Consider, for example, hierarchical clustering algorithms which generate binary clustering trees. Clustering trees are useful tools for dealing with real-world data visualization problems. Our algorithm, the \emph{Big Data Visualization Tool}, or \mbox{\textit{BiDViT}}, provides this functionality.

BiDViT employs a novel approach to clustering problems, which is based on the maximum weighted independent set (MWIS) problem in a graph induced by the original dataset and a parameter we call the \emph{radius of interest} or \emph{neighbourhood parameter}, which determines a relation of proximity. The use of such a parameter has been successfully employed in density-based spatial clustering of applications with noise (DBSCAN) \cite{ester1996density}. The MWIS problem can be transformed into a quadratic unconstrained binary optimization (QUBO) problem, the formulation accepted by a quantum annealer. An alternative way to address the underlying problem is to use a heuristic algorithm to approximate solutions to the MWIS problem. Quantum annealing and simulated annealing have been applied in centroid-based clustering \cite{kumar2018quantum, merendino2013simulated} and in density-based clustering \cite{kurihara2014quantum}. However, the approaches studied are not capable of addressing problems in the extreme clustering domain.

We prove that, under a separability assumption on the ground truth clustering assignment of the original dataset, our method identifies the ground truth labels when  parameters are selected that are within the bounds determined by that assumption. We provide runtime and solution quality values for both  versions of our algorithm, with respect to internal evaluation schemes such as the \mbox{Calinski--Harabasz} and the Davies--Bouldin scores. Our results suggest that \mbox{BiDViT} yields clustering assignments of a quality comparable to that of assignments generated by common clustering algorithms, yet does so a full order of magnitude faster.

\section{The Coarsening Method}
\label{sec:approximation_method}
Our algorithm is based on a combinatorial clustering method we call \mbox{\textit{coarsening}}. The key idea behind coarsening is to approximate a set \mbox{$X \subset \mathbb{R}^d$} by a subset $S \subseteq X$ such that, for any point \mbox{$x \in X$}, there exists a $y \in S$ such that \mbox{$\Vert x -y \Vert_2 < \epsilon$}, for some parameter $\epsilon > 0$. In this case, we say that $S$ is \mbox{$\epsilon$-\emph{dense}} in $X$ and call $\epsilon$ the \emph{radius of interest}. This concept is not restricted to subsets of Euclidean spaces and can be generalized to an arbitrary metric space $(M,d)$. For example, our coarsening method can be used for clustering assignments on finite subsets of Riemannian manifolds with respect to their geodesic distance, for instance, in clustering GPS data on the surface of the Earth when analyzing population density. In what follows, we assume that $X = \lbrace x^{(1)}, \ldots, x^{(n)} \rbrace$ is a dataset consisting of $n$ $d$-dimensional data points, equipped with a metric $d: X \times X \rightarrow [0,\infty)$. Finding an arbitrary \mbox{$\epsilon$-dense} subset of $X$ does not necessarily yield a helpful approximation. For example, $X$ itself is always $\epsilon$-dense in $X$. However, enforcing the additional constraint that any two points in the subset $S$ must be separated by a distance of at least $\epsilon$ yields more-interesting approximations, often leading to a reduction in the number of data points (one of our primary objectives). We call such a set \mbox{$\epsilon$-\emph{separated}}. \cref{fig:voronoi} shows a point cloud and an $\epsilon$-dense, $\epsilon$-separated subset. The theorem that follows shows that a maximal $\epsilon$-separated set $S$ of $X$ is necessarily $\epsilon$-dense in $X$. Let $B(x,r)$ denote the open metric ball with respect to $d$, with centre $x$ and radius $r$.

\begin{figure}[t]
  	\centering
  	\hfill
  	\includegraphics[clip, trim=2.4cm 1.5cm 1.8cm 1.3cm, width=0.4\textwidth]{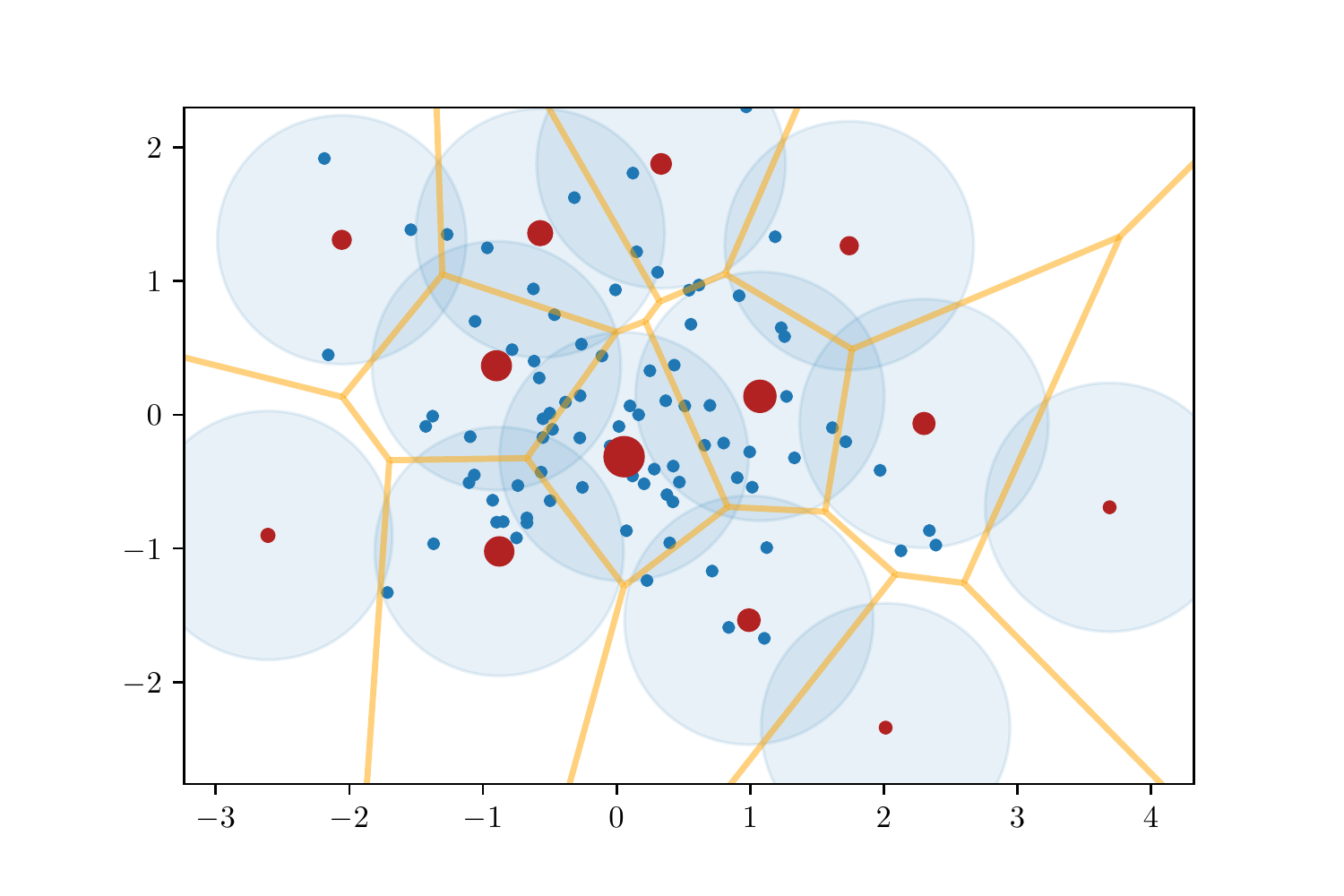}
  	\hfill
  	\includegraphics[clip, trim=3cm 3cm 3cm 2.5cm, width=0.5\textwidth]{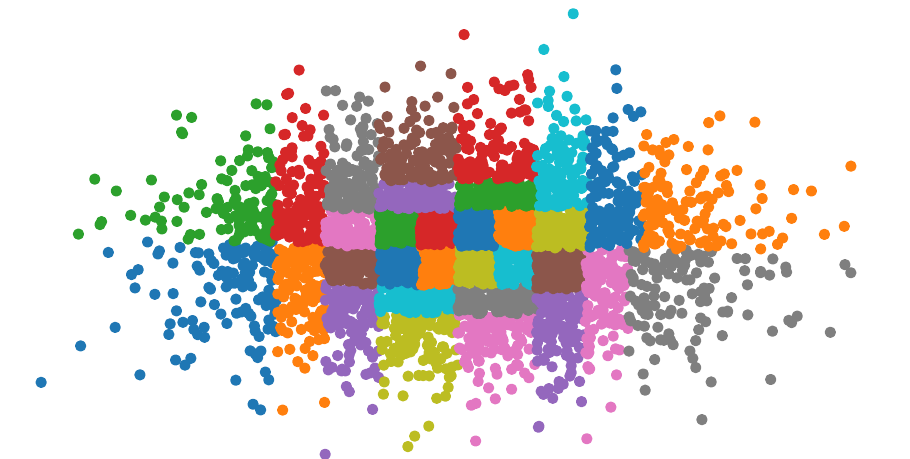}
  	\hfill
    \caption{Visualization of chunk collapsing (left) and data partitioning (right). Left) A maximal $\epsilon$-separated subset (red dots) of a dataset (red dots and blue dots). The circles have a radius equal to the radius of interest $\epsilon$. The weights of the red points are updated according to the number of blue points within a distance of $\epsilon$. The yellow borders are a Voronoi partition of the dataset indicating the clustering assignment. Right) Data partitioning of a dataset along the axes of maximum variance. In this example, there are $s=5$ partitioning steps, resulting in $2^5=32$ chunks.}
    \label{fig:voronoi}
    \label{fig:chunking}
\end{figure}

\begin{theorem}
\label{thm:seperated_implies_dense}
Let $S$ be a maximal $\epsilon$-separated subset of $X$ in the sense of set inclusion. Then the following properties must be satisfied.
\begin{itemize}
    \item[i)] We have the inclusion $X \subseteq \bigcup_{x \in S} B(x,\epsilon)$.
    \item[ii)] For every $y \in S$, it holds that $X \nsubseteq \bigcup_{x \in S\setminus \lbrace y \rbrace} B(x,\epsilon)$.
    \item[iii)] The sets $ B(x,\epsilon/2)$ for $x \in S$ are pairwise disjoint.
\end{itemize}
In particular, $S$ is a minimal $\epsilon$-dense subset of $X$.
\end{theorem}

\begin{proof}
Note that $i)$ is equivalent to $S$ being $\epsilon$-dense in $X$ and that, in combination with $ii)$, is equivalent to $S$ being a minimal with respect to this property. To prove $i)$, let $S$ be a maximal $\epsilon$-separated subset of $X$ and assume, in contradiction, that $S$ is not $\epsilon$-dense in $X$. Then we could find $x \in X$ such that $d(x,y) \geq \epsilon$, for every $y \in S$. Hence, $S \cup \lbrace x \rbrace$ would be $\epsilon$-separated, which is in contradiction to the maximality of $S$. To prove $ii$), we fix a point $x \in S$. Since $S$ is $\epsilon$-separated, $d(x,y) \geq \epsilon$ for any $y \in S$ and, thus, $S \setminus \lbrace x \rbrace$ is not $\epsilon$-dense in $X$. Property $iii$) follows from the triangle inequality.\qed
\end{proof}

Note that a maximal $\epsilon$-separated subset does not refer to an $\epsilon$-separated subset with fewer than or equally as many elements as all other $\epsilon$-separated subsets but, rather, to an $\epsilon$-separated subset that is no longer $\epsilon$-separated when a single data point is added. Contrary to~ \cref{thm:seperated_implies_dense}, a minimal \mbox{$\epsilon$-dense} subset does not need to be $\epsilon$-separated. Consider the set \mbox{$X = \lbrace 1,2,3,4 \rbrace \subset \mathbb{R}$}, and let $d$ be the Euclidean distance on $\mathbb{R}$. Then, $S=\lbrace 2,3 \rbrace$ is \mbox{$\sfrac{3}{2}$-dense} in $X$ but not \mbox{$\sfrac{3}{2}$-separated}. Also note that an $\epsilon$-separated subset is not necessarily an $\epsilon$-coreset, which is a weighted subset whose weighted $k$-means cost approximates the $k$-means cost of the original set with up to an accuracy of $\epsilon$ \cite{har2004coresets, balcan2013distributed}.

In the following, we assume that $X$ is equipped with a \emph{weight function} \mbox{$w: X \rightarrow \mathbb{R}_{+}$}. We call $w_i=w(x^{(i)})$ the \emph{weight} of $x^{(i)}$ and gather all weights in a \emph{weight vector} $w \in \mathbb{R}^n_+$. It will be clear from the context whether we refer to a weight function or a weight vector. The weight of a set $S \subseteq X$ is given by $\omega(S) = \sum_{x \in S} w(x)$. We have already argued that maximal $\epsilon$-separated subsets yield reasonable approximations. However, such subsets are not unique. We are thus  interested in finding an \emph{optimal} one, that is,  one that captures most of the weight of the original dataset. In other words, we are interested in solving the optimization problem
\begin{equation}
    \label{prob:original_problem}
    \tag{P\arabic{prob}}
    \addtocounter{prob}{1}
    \underset{S \subseteq X}{\text{maximize}}\,\,\,
    \omega(S) \quad
    \text{subject to} \quad
    S \text{ is $\epsilon$-separated}.
\end{equation}
If we impose unit weights, the solution set to this optimization problem will consist of the maximal $\epsilon$-separated subsets of $X$ with a maximum number of elements among all such subsets. The term ``maximal'' refers to set inclusion and the ``maximum'' refers to set cardinality. Since $w(x) >0$ for all $x \in X$, a solution $S^*$ to \cref{prob:original_problem} will always be a maximal $\epsilon$-separated subset and, therefore, by \cref{thm:seperated_implies_dense}, $\epsilon$-dense.
In \cref{sec:relation}, we show that this problem is equivalent to solving an MWIS problem for a weighted graph $G^\epsilon(X,E^\epsilon,w)$, depending solely on the dataset $X$, the Euclidean metric $d$, and the radius of interest $\epsilon$. Thus, the computational task of finding a maximal $\epsilon$-separated subset of maximum weight is NP-hard~\cite{lucas2014ising, karp1972reducibility}.

Every subset $U \subset X$ gives rise to a \emph{clustering assignment} \mbox{$\cC= \lbrace C_x \rbrace_{x \in U}$}. This assignment is given by 
\begin{equation}
    \label{eq:cluster_assignment}
    C_x = \lbrace y \in X : d(x,y) \leq d(x',y) \text{ for all } x' \in U \rbrace.
\end{equation}
Data points that are equidistant to multiple representative points are assigned to only one of them, uniformly at random. Typically, larger  values of $\epsilon$ result in smaller cardinalities of $\cC$. The following corollary summarizes properties of $\cC$ when $U$ is $\epsilon$-separated, and can be  readily verified.

\begin{corollary}
Let $\cC$ be the clustering assignment generated from a maximal \mbox{$\epsilon$-separated} set $S \subset X$. Then, the following properties are satisfied:
\begin{itemize}
    \item[i)] The clusters in $\cC$ are non-empty and pairwise disjoint.
    \item[ii)] The cluster diameter is uniformly bounded by $2\epsilon$, i.e., $\sup_{x \in S} \textup{diam}(C_x) \leq 2\epsilon$.
    \item[iii)] For all $x \in S$, it holds that $\max_{y \in C_x} d(x,y) < \epsilon$.
\end{itemize}
\end{corollary}
Notice that these properties are not satisfied by \emph{every} clustering assignment, for example, the ones generated by $k$-means clustering. They are desirable in specific applications, such as image quantization, where a tight bound on the absolute approximation error is desired. However, they are undesirable if the ground truth clusters have diameters larger than $2\epsilon$. More details on the clustering assignment are provided in \cref{sec:algorithm}. 

One could argue that prior to identifying a maximum weighted independent set and using it to generate a clustering assignment, a dataset should be normalized. However, normalization is a transformation that would result in chunks not being defined by metric balls, but rather by ellipsoids. In particular, such a transformation would change the metric $d$. We assume that the metric $d$ already is the best indicator of proximity. 
In general, one can apply any homeomorphism $f$ to a dataset $X$, apply our clustering algorithm to the set $f(X)$, and obtain a clustering assignment by applying $f^{-1}$ to the individual clusters.

A common assumption in the clustering literature is \emph{separability}---not to be mistaken with \mbox{$\epsilon$-separability}---of the dataset with respect to a clustering $\mathcal{C}$. The dataset $X$ is called \emph{separable} with respect to a clustering $\mathcal{C}= \lbrace C_1, \ldots C_k \rbrace$ if
\begin{equation}
    \label{eq:sep}
    \max_{\substack{x,y \in C_i \\ 1 \leq i \leq k}} d(x,y) < \min_{\substack{x \in C_i, y \in C_j \\ 1 \leq i \neq j \leq k }} d(x,y),
\end{equation}
that is, if the maximum \emph{intra-cluster} distances are strictly smaller than the minimum \emph{inter-cluster} distances.
The following theorem shows that, if $\epsilon$ is chosen correctly, our coarsening method yields the clustering assignment $\mathcal{C}$. 

\begin{theorem}
    \label{thm:separability}
    Let $X$ be separable with respect to a clustering $\mathcal{C}= \lbrace C_1, \ldots C_k \rbrace$. Then, for any 
    \begin{equation}
            \label{eq:eps_range}
            \epsilon \in \left(\max_{\substack{x,y \in C_i \\ 1 \leq i \leq k}} d(x,y) , \min_{\substack{x \in C_i, y \in C_j \\ 1 \leq i \neq j \leq k }} d(x,y)\right],
    \end{equation}
    our coarsening methods yields the correct clustering assignment.
\end{theorem}

\begin{proof}
    To simplify our notation, we denote the lower and upper bounds of the interval in \eqref{eq:eps_range} by $l$ and $r$, respectively. By the separability assumption, this interval is non-empty. One can see that, for any admissible choice of $\epsilon$, any two points from different clusters are $\epsilon$-separated. Indeed, for $x \in C$ and $y \in C'$, it holds that
    $
        d(x,y) \geq r \geq \epsilon
    $.
    Furthermore, if a point $x$ in a cluster $C$ is selected, then no other point $y$ in the same cluster can be selected, as
    $
        d(x,y)  \leq l < \epsilon
    $.
    Therefore, every solution $S \subseteq X$ to  \cref{prob:original_problem} is a union of exactly one point from each cluster. Using the separability of $X$ with respect to $\mathcal{C}$, we can see that the clustering assignment induced by~\cref{eq:cluster_assignment} is coincident with $\mathcal{C}$. \qed
\end{proof}

In practice, the separability assumption is rarely satisfied, and it is challenging to select $\epsilon$ as above (as this assumes some knowledge about the clustering assignment). However, \cref{thm:separability} shows that our coarsening method is of research value, and can potentially yield optimal clustering assignments.

We have developed two methods, which we refer to as the \emph{heuristic method} and the \emph{quantum method}, to address the NP-hard task of solving \cref{prob:original_problem}. The heuristic method loosens the condition of having a maximum weight; it can be seen as a greedy approach to \cref{prob:original_problem}. In contrast, the quantum method explores all different maximal \mbox{$\epsilon$-separated} subsets simultaneously,  yielding one that has maximum weight. The quantum method is based on the formulation of a QUBO problem, which can be solved efficiently using a quantum annealer like the \mbox{D-Wave 2000Q} \cite{DWave} or a digital annealer such as the one developed by \mbox{Fujitsu~\cite{Fujitsu}}. 

\section{The Algorithm}
\label{sec:algorithm}

Let \mbox{$X = \lbrace x^{(1)},\ldots,x^{(n)} \rbrace \subset \mathbb{R}^d$} denote a dataset of $n$ $d$-dimensional data points. Note that, mathematically speaking, a dataset is not a set but rather a \emph{multiset}, that is, repetitions are allowed.  Our algorithm consists of two parts:  \emph{data partitioning} and  \emph{data coarsening}, the latter of which can be further subdivided into \emph{chunk coarsening} and \emph{chunk collapsing}.

\subsection{Data Partitioning}
\label{sec:data_partitioning}
In general, the computational complexity of distance-based clustering methods is proportional to the square of the dataset cardinality, as all pairwise distances must be computed. This bottleneck can be overcome by dividing the dataset and employing distributed approaches \cite{balcan2013distributed, malkomes2015fast, har2004coresets}, yielding a result different from the one we would obtain when applying clustering methods on the entire dataset. However, its slight imprecision results in a significant computational speed-up. 

A \emph{partition} $\mathcal{P}$ of $X$ is a collection of non-empty disjoint sets \mbox{$P_1, \ldots, P_k \subset X$} such that $X = \bigcup_{P \in \mathcal{P}} P$. Elements of partitions are typically referred to as blocks, parts, or cells; however, we refer to them as \emph{chunks}. The partitioning is intended to be \textit{homogeneous}: every extracted chunk has an equal number of data points (there might be minor differences when the cardinality of the chunk to be divided is odd). The upper bound on the number of points desired in a chunk is referred to as the \textit{maximum chunk cardinality} $\kappa$. To determine $\kappa$, one should take into account the number of available processors, their data handling capacity, or, in the case of a quantum annealer, the number of fully connected qubits.

To break the data into chunks, we employ a modified version of the well-known ``median cut" algorithm, which is frequently used in colour quantization. First, we select an axis of maximum variance. We then bisect the dataset along the selected axis, say $\ell$, at the median $m$ of $\lbrace x^{(1)}_\ell, \ldots, x^{(n)}_\ell \rbrace$ in such a way as to  obtain two \emph{data chunks} $P_1$ and $P_2$ whose cardinalities differ by at most one (in the case where $n$ is odd) and which satisfy
\mbox{$P_1 \subseteq \lbrace x \in X  :  x_\ell \leq m \rbrace$} 
and
\mbox{$P_2 \subseteq \lbrace x \in X  :  x_\ell \geq m \rbrace$}.
We cannot simply assign \mbox{$P_1 = \lbrace x \in X  :  x_\ell \leq m \rbrace$} and $P_2 = X \setminus P_1$, as these sets might differ drastically in cardinality. For example, when $x_\ell^{(1)} = \ldots = x_\ell^{(n)}$, this assignment would imply that $P_1 = X$ and $P_2 = \emptyset$. 

By using $P_1$ and $P_2$ in the role of $X$, this process can be repeated iteratively, until the number of data points in the chunk to be divided is less than or equal to the maximum chunk cardinality $\kappa$, yielding a binary tree of data chunks. After $s$ iterations, this leaves us with $2^s$ chunks $P^{(s)}_k$ such that
$
    X = \bigcup_{1 \leq k \leq 2^s} P^{(s)}_k,
$
where the union is disjoint. \cref{fig:chunking} provides a visualization.

\subsection{Chunk Coarsening}
\label{sec:chunk_approx}

The goal of a data coarsening step is, for each chunk, to find \textit{representative} data points such that their union can replace the original point cloud, while maintaining the original data distribution as accurately as possible. 

Let $P = \lbrace x^{(1)}, \ldots, x^{(n)} \rbrace$ be a chunk and $\epsilon>0$ be the radius of interest. In what follows, we assume that all the data points are pairwise different. Practically, this can be achieved by removing duplicates and cumulatively incrementing the weight of the  representative point we wish to keep by the weight of the discarded duplicates. The radius of interest~$\epsilon$ induces a weighted graph $G^\epsilon=(P, E^\epsilon, w_P)$, where $P$ is the vertex set, the edge set $E^\epsilon$ is given by the relation~$\sim_\epsilon$ defined by $x \sim_\epsilon y$ if and only if $d(x,y) < \epsilon$ for all $x,y \in P$, and the weight function $w_P: P \rightarrow \mathbb{R}_+$ is the restriction of $w$ to $P$. For each data point $x^{(i)}$, we denote its weight $w_P(x^{(i)})$ by $w_i$.

For each data point $x^{(i)}$, we introduce a binary decision variable $s_i$ that encodes whether $x^{(i)}$ is used in a possible set $S^*$. Furthermore, we define the  \emph{neighbourhood matrix} $N^{(\epsilon)}$ (or similarity matrix) of the graph $G^\epsilon=(P,E^\epsilon, w_P)$ by  $N_{ij}^{(\epsilon)} = 1$ if  \mbox{$x^{(i)} \sim_\epsilon x^{(j)}$}, and $N_{ij}^{(\epsilon)} = 0$ otherwise. Problem \cref{prob:original_problem} can then be posed as a quadratically constrained quadratic program (QCQP) given by
\begin{equation}
    \tag{P\arabic{prob}}
    \addtocounter{prob}{1}
    \label{prob:main_problem}
    \underset{s \in \lbrace 0,1 \rbrace^n}{\text{maximize}}\,\,
    \sum_{i=1}^n s_i w_i \quad
    \text{subject to} \quad
    \sum_{i=1}^n \sum_{j > i} s_i N_{ij}^{(\epsilon)} s_j = 0.
\end{equation}
Here, the inner summation of the constraint does not need to run over all indices, due to the symmetry of $N^{(\epsilon)}$. The matrix form of \cref{prob:main_problem} is given by maximizing
$ s^T w $ subject to the constraint $s^T \overline{N}^{(\epsilon)} s = 0$, where $\overline{N}^{(\epsilon)}$ is the upper triangular matrix of $N^{(\epsilon)}$ having all zeroes along the diagonal. As explained in \cref{sec:relation},  \cref{prob:main_problem} is equivalent to the NP-hard MWIS problem for $G^\epsilon=(P,E^\epsilon, w_P)$, and thus is computationally intractable for large problem sizes. Note that \cref{prob:main_problem} can be written as the \mbox{0--1 integer linear program (ILP)}
\begin{equation}
    \tag{P\arabic{prob}}
    \addtocounter{prob}{1}
    \label{prob:main_problem_linearization}
    \underset{s \in \lbrace 0,1 \rbrace^n}{\text{maximize}}\,\,
    \sum_{i=1}^n s_i w_i \quad
    \text{subject to} \quad
    s_i + s_j \leq 1, \quad \text{for } i,j \text{ such that } \overline{N}_{ij}^{(\epsilon)} = 1.
\end{equation}

\noindent
We present two methods we have devised to address   \cref{prob:main_problem}.

\subsubsection{The Heuristic Method}
We wish to emphasize that the heuristic method does not provide us with a solution to \cref{prob:main_problem}. Rather, the aim of this method is to obtain an \mbox{$\epsilon$-separated} subset $S$ with a high---but not necessarily the maximum---weight~$\omega(S)$. The seeking of approximate solutions to the MWIS problem is a well-studied subject \cite{balaji2009approximating, hifi1997genetic, kako2009approximation}. Typically, researchers employ greedy algorithms, LP-based algorithms (using the relaxation of \cref{prob:main_problem_linearization}), or semi-definite programming (SDP) algorithms; see \cite{kako2009approximation} for an analysis. 

We employ a classic greedy algorithm due to its simplicity and low computational complexity. In each step, we add the data point that \emph{locally} is the best choice in the sense that the ratio of the weight of its neighbourhood to its own weight is as small as possible. Prior to the execution of the step, we remove the point and its neighbours from the set of candidates. Pseudocode of the greedy algorithm is provided in \cref{alg:heuristic}. Before we state a theoretical result on the approximation ratio of this algorithm, we define the \emph{weighted degree} $\deg_w(v)$ of a vertex $v$ in a weighted graph $G=(V,E,w)$ and the \emph{weighted average degree} of $G$ as
$
    \deg_w(v) = \omega(N_v) / w(v)
$ 
and 
$
    \overline{\deg_w}(G) = \sum_{v \in V}w(v)\deg_w(v) / \omega(V),
$
respectively, where $N_v = \lbrace u \in V : u \sim v \rbrace$ is the neighbourhood of vertex $v$ \cite{kako2009approximation}.

\RestyleAlgo{ruled}
\begin{algorithm}[t]
\DontPrintSemicolon
{\scriptsize
\SetKwInOut{Input}{input}\SetKwInOut{Output}{output}
\Input{data chunk $P$; weight function $w$; neighbourhood matrix $N^{(\epsilon)}$}
\Output{$\epsilon$-separated subset with high (but not necessarily maximum) weight $S^*$}
$S^* \leftarrow \emptyset$\;
\While{$P \neq \emptyset$}{
select $x \in \argmin_{v \in P} \deg_w(v) $ uniformly at random\; 
use $N^{(\epsilon)}$ to determine $N_x$\;
remove $x$ and its neighbours $N_x$ from $P$ \; 
$S^* \leftarrow S^* \cup \lbrace x \rbrace$\;
}
\Return{$S^*$
}
}
\caption{\, Greedy$(P,w,N^{(\epsilon)})$}
\label{alg:heuristic}
\end{algorithm}

\begin{theorem}
    \label{thm:approximation_ration}
    \Cref{alg:heuristic} has an approximation ratio of $\overline{\deg_w}(G) + 1$, i.e., 
    \begin{equation}
    \label{eq:approximation_ratio}
        \omega(S) \leq \left(\overline{\deg_w}(G) + 1 \right)^{-1} \omega(S^*),
    \end{equation}
    for any solution $S^*$ to \cref{prob:original_problem} and any output $S$ of the algorithm. Moreover, the bound in \eqref{eq:approximation_ratio} is tight.
\end{theorem}

\begin{proof}
    A proof is given in \cite[Thm. 6]{kako2009approximation}. \qed
\end{proof}

\subsubsection{The Quantum Method}
In contrast to the heuristic method, the QUBO approach provides an actual (i.e., non-approximate) solution to \cref{prob:main_problem}. We reformulate the problem by transforming the QCQP into a QUBO problem.

Using the Lagrangian penalty method, we incorporate the constraint into the objective function by adding a penalty term. For a sufficiently large penalty multiplier $\lambda>0$, the solution set of \cref{prob:main_problem} is equivalent to that of
\begin{equation}
    \tag{P\arabic{prob}}
    \addtocounter{prob}{1}
    \label{prob:main_problem_QUBO}
    \underset{s \in \lbrace 0,1 \rbrace^n}{\text{maximize}}\,\,
    \sum_{i=1}^n s_i w_i - \lambda\sum_{i=1}^n \sum_{j > i} s_i N_{ij}^{(\epsilon)} s_j.
\end{equation}

One can show that, for $\lambda > \max_{i=1, \ldots n}, w_i$ every solution to \cref{prob:main_problem_QUBO} satisfies the separation constraint \cite[Thm. 1]{abbott2018hybrid}. Instead, we use individual penalty terms $\lambda_{ij}$, as this may lead to a QUBO problem with much smaller coefficients, which results in improved performance when solving the problem using a quantum annealer. Expressing \cref{prob:main_problem_QUBO} as a minimization, instead of a maximization, problem and using matrix notation yields the problem
\begin{equation}
    \tag{P\arabic{prob}}
    \addtocounter{prob}{1}
    \label{prob:main_problem_QUBO_matrix}
    \underset{s \in \lbrace 0,1 \rbrace^n}{\text{minimize}}\,\,
    s^T Q s,
\end{equation}
where $Q_{ij}=-w_i$ if $i=j$, $Q_{ij}=\lambda_{ij}$ if $N_{ij}^{(\epsilon)}=1$ and $i<j$, and $Q_{ij}=0$ otherwise. 
Solutions to \cref{prob:main_problem_QUBO_matrix} can be approximated using heuristics such as simulated annealing \cite{nolte2000note}, path relinking~\cite{lu2010hybrid}, tabu search \cite{lu2010hybrid}, and parallel tempering \cite{zhu2016borealis}. Before solving \cref{prob:main_problem_QUBO_matrix}, it is advisable to reduce its size and difficulty by making use of logical implications among the coefficients \cite{glover2018logical}. This involves fixing every variable that corresponds to a node that has no neighbours to one, as it  necessarily is included in an $\epsilon$-dense subset. 

The following theorem show that \cref{prob:main_problem} is equivalent to \cref{prob:main_problem_QUBO_matrix} for a suitable choice of $\lambda_{ij}$, for $1\leq i<j \leq n$.

\begin{theorem}
    \label{thm:multiplier}
    Let $\lambda_{ij}  > \max \lbrace w_i, w_j \rbrace$ for all $1 \leq i < j \leq n$. Then, for any solution $s \in \lbrace 0, 1 \rbrace^n$ to \cref{prob:main_problem_QUBO_matrix}, the corresponding set $S \subseteq X$ is \mbox{$\epsilon$-separated}. In particular, the solution sets of \cref{prob:main_problem} and \cref{prob:main_problem_QUBO_matrix} coincide.
\end{theorem}

\begin{proof}
    We generalize the proof of  \cite[Thm. 1]{abbott2018hybrid} and show that every solution $s$ to \cref{prob:main_problem_QUBO_matrix} satisfies the separation constraint
    $
        \sum_{i=1}^n \sum_{j>i} s_i N_{ij}^{(\epsilon)} s_j = 0
    $.
    Assuming, in contradiction, that the opposite were to be the case, we could find a solution $s$ and indices $k$ and $\ell$ such that $1 \leq k < \ell \leq n$ and $s_k = s_\ell = N_{k\ell}^{(\epsilon)} =1$. Let $e_k$ denote the $k$-th standard unit vector, and let $v = s - e_k$. Then,
    \begin{align}
         v^T Q v    &= s^T Q s  - \sum_{j>k}^n s_j Q_{kj} - \sum_{i<k}^n s_i Q_{ik} -  Q_{kk} \\
                    &= s^T Q s  - \sum_{i\neq k} s_i \lambda_{\sigma(i,k)} N_{ik}^{(\epsilon)} + w_k,
    \end{align}
    where $\sigma: \mathbb{N}^2 \rightarrow \mathbb{N}^2$, defined by $\sigma(i,k)=(\min(i,k),\max(i,k))$, orders the index accordingly. This technicality is necessary, as we defined $\lambda_{ij}$ only for \mbox{$1 \leq i <j \leq n$}. As \mbox{$N_{k\ell}^{(\epsilon)} = s_\ell = 1$}, we have $\sum_{i\neq k}s_i \lambda_{\sigma(i,k)} N_{ik}^{(\epsilon)} \geq \lambda_{\sigma(\ell,k)} $, and thus
    \begin{equation}
        v^T Q v \leq s^T Q s - \lambda_{k \ell} + w_k.
    \end{equation}
    Therefore, as $\lambda_{k \ell} > \max \lbrace{ w_k, w_\ell \rbrace} \geq w_k$, it holds that $v^T Q v < s^T Q s$, which is absurd, as, by assumption, $s$ is a solution to~\cref{prob:main_problem_QUBO_matrix}.

    We now show that the solution sets of \cref{prob:main_problem} and \cref{prob:main_problem_QUBO_matrix} coincide. Note that \cref{prob:main_problem} is equivalent to the optimization problem
    \begin{equation}
    \tag{P\arabic{prob}}
    \addtocounter{prob}{1}
    \label{eq:same_form}
    \underset{s \in \lbrace 0,1 \rbrace^n}{\text{minimize}}\,
,\    - s^T w \quad
    \text{subject to} \quad
    s^T \overline{N}^{(\epsilon)} s = 0.
\end{equation}
   Let $s_1$ and $s_2$ be solutions to \cref{eq:same_form} and \cref{prob:main_problem_QUBO_matrix}, respectively. We denote the objective functions by $p_1(s) = -s^T w$ and  $p_2(s) = -s^T w + s^T \left(\Lambda \circ \overline{N}^{(\epsilon)}\right) s$,
   where $\Lambda$ is the matrix defined by $\Lambda_{ij} = \lambda_{ij}$ for $1 \leq i < j \leq n$, and zero otherwise, and the term $\Lambda \circ \overline{N}^{(\epsilon)} \in \mathbb{R}^{n \times n}$ denotes the Hadamard product of the matrices $\Lambda$ and $\overline{N}^{(\epsilon)}$, given by
   element-wise multiplication.
   Then, as $\lambda_{ij} > \max \lbrace w_i,w_j \rbrace$ for $1 \leq i < j \leq n$, by the observation above, both $s_1$ and $s_2$ satisfy the separation constraint. Since $s$ and $\overline{N}^{\epsilon}$ are coordinate-wise non-negative and $\lambda_{ij}> \min_{k=1,\ldots,n} w_k > 0$ for $1 \leq i < j \leq n$, it holds that
   \begin{equation}
       s^T \overline{N}^{(\epsilon)}s = 0 \quad \Leftrightarrow \quad 
       s^T \left(\Lambda \circ \overline{N}^{(\epsilon)}\right) s=0,
   \end{equation}
   thus, if $s$ satisfies the separation constraint, then $p_2(s) = p_1(s)$. Using this observation, and that $s_1$ and $s_2$ minimize $p_1$ and $p_2$, respectively, we have 
   \begin{equation}
   \label{eq:chain}
       p_1(s_1) \leq p_1(s_2) = p_2(s_2) \leq p_2(s_1) = p_1(s_1).
   \end{equation}
   Hence, the inequalities in \cref{eq:chain} must actually be equalities; thus, the solution sets of the optimization problems coincide. \qed
\end{proof}

Problem \cref{prob:main_problem_QUBO_matrix} can be transformed to an Ising spin model by mapping $s$ to $2s-1$. This form is desirable because the ground state of the Hamiltonian of an Ising spin model can be determined efficiently with a quantum annealer. 

\subsection{Chunk Collapsing}
Having identified a maximal $\epsilon$-separated subset $S \subseteq P$, we \emph{collapse} the vertices $P \setminus S$ into $S$, meaning we update the weight of each $x \in S$ according to the weights of all $y \in P \setminus S$ that satisfy $x \sim_\epsilon y$. We aim to assign each $y \in P \setminus S$ to a \emph{unique} $x \in S$ by generating a so-called Voronoi decomposition (depicted in \cref{fig:voronoi}) of each chunk $P$, which is a partition, where each point $x \in P$ is assigned to the closest point within a subset $S$. More precisely, we define the sets $C_x$ as in \cref{eq:cluster_assignment}, for each $x \in S$. By construction, $C_x$ contains all vertices that  will be collapsed into $x$, in particular, $x$ itself. We then assign the coarsened chunk $S$ a new weight function $w_S$ defined by
\begin{equation}
    w_S(x) = \omega(C_x) = \sum_{y \in C_x} w(y).
\end{equation}
In practice, to prevent having very large values for the individual weights, one might wish to add a linear or  logarithmic scaling to this weight assignment. In our experiments, we did not add such a scaling.

\subsection{Iterative Implementation of BiDViT}
BiDViT repeats the procedure of data partitioning, chunk coarsening, and chunk collapsing with an increasing radius of interest, until the entire dataset collapses to a single data point. We call these iterations BiDViT~\textit{levels}. The increase of  $\epsilon$  between BiDViT levels is realized by multiplying $\epsilon$ by a constant factor, denoted by $\alpha$ and specified by the user. In our implementation we have introduced a \verb|node| class that has three attributes: \verb|coordinates|, \verb|weight|, and \verb|parents|. We initialize BiDViT by creating a \verb|node_list| containing the nodes corresponding to the weighted dataset (if no weights are provided then the weights are assumed to be the multiplicity of the data points). After each iteration, we remove the nodes that collapsed into representative nodes from the \verb|node_list| and keep only the remaining representative nodes. However, we append the removed nodes to the parents of the representative node. The final \verb|node_list| is a data tree, that is, it consists of only one node, and we can move upwards in the hierarchy by accessing its parents (and their parents and so on); see \cref{fig:tree}. Two leaves of the data tree share a label with respect to a specific BiDViT level, say $m$, if they have collapsed into centroids which, possibly after multiple iterations, have collapsed into the same centroid at the $m$-th level of the tree. For the sake of reproducibility, we provide pseudocode (see~\cref{alg:bidvit}).

\RestyleAlgo{ruled}
\begin{algorithm}[t]
\DontPrintSemicolon
{\scriptsize
\SetKwInOut{Input}{input}\SetKwInOut{Output}{output}
\Input{data set $X$; initial radius $\epsilon_0$; maximum chunk cardinality $\kappa$; radius increase rate $\alpha$}
\Output{tree structure that encodes the hierarchical clustering $\cT$}
$ \cT \leftarrow  \textup{create\_node\_list}(X)$ \;
$\epsilon \leftarrow \epsilon_0$ \;
\While{$\textup{length}(\cT) > 1$}{
$\cP \leftarrow$ partition($\cT, \kappa$) \;
$\cT \leftarrow \emptyset$ \;
\For{$P \in \cP$}{
    compute neighbourhood matrix $N^{(\epsilon)}$ for P\;
    identify representive data points by solving MWIS for $P, N^{(\epsilon)},$ and $w$ \;
    compute Voronoi partition of $P$ with respect to representative points \;
    compute centroids of the cells of the Voronoi partition\;
    \For{$x \in P$}{
        ind $\leftarrow$ closest\_centroid($x$, centroids) \;
        centroids[ind].weight += $x$.weight \;
        centroids[ind].parents.append($x$) \;
    }
    $\cT$.append(centroids) \;
}
$\epsilon \leftarrow \alpha \epsilon$ \;}
\Return{$\cT$
}
}
\caption{\, BiDViT}
\label{alg:bidvit}
\end{algorithm}

It is worth noting that, at each level, instead of proceeding with the identified representative data points, one can use the cluster centroids, allowing more-accurate data coarsening and label assignment. 

\begin{figure}[t]
    \centering
    \includegraphics[clip, trim=4.8cm 15.3cm 2.2cm 4.4cm, width=0.78\textwidth]{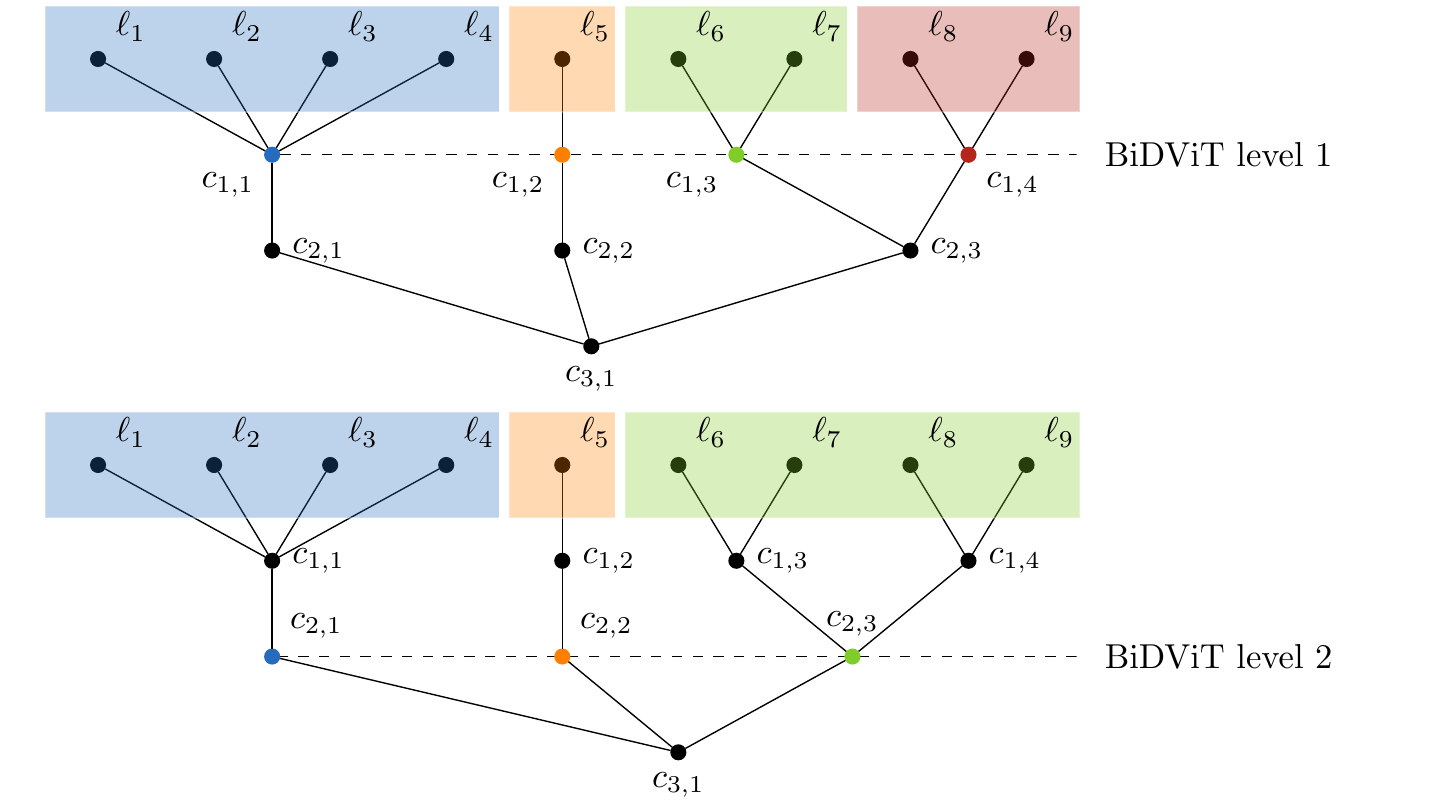}
    \caption{Dendrogram representing the output tree of BiDViT and the encoding of the clustering assignment. The original dataset is represented by the leaves, which collapse into a single centroid after three BiDViT iterations. The first iteration (``BiDViT level 1'') results in four centroids, each corresponding to a cluster consisting of the nodes that collapsed into it. At the next iteration, the algorithm merges the clusters of the centroids. For example, $c_{1,3}$ and $c_{1,4}$ are merged into $c_{2,3}$ at the next level.}
    \label{fig:tree}
\end{figure}

\subsection{Complexity Analysis}
\label{sec:complexity}
Our analysis shows that every interation of the heuristic version of  BiDViT has a computational complexity of \mbox{$\bigO(dn\log(n/\kappa)+dn\kappa)$}. Note that $\kappa \ll n$. 

The order of complexity of the partitioning procedure is $\bigO(d n \log(n/ \kappa))$. To see this, note that there are at most $\lceil \log_2(n/\kappa) \rceil$ partitioning stages and in the $s$-th stage we split $2^{s-1}$ chunks $P_i$, where $i=1,\ldots,2^{s-1}$. Let $n_i$ denote the number of data points in chunk $P_i$. Finding the dimension of maximum variance has a complexity of $\bigO(d n_i)$ and determining the median of this dimension can be achieved in $\bigO(n_i)$ via the  ``median of medians'' algorithm. Having computed the median, one can construct two chunks of equal size in linear time. Since $\sum_{1 \leq i \leq 2^{s-1}}n_i = n$, a partitioning step is \mbox{$ \bigO(d n\log(n / \kappa))$}. Any division of a chunk is independent of the other chunks at a given stage;  thus, this procedure can benefit from distributed computing.

The order of complexity for the collapsing process is $\bigO(d n \kappa)$, as computing the neighbourhood matrix of a chunk is  $\bigO(d\kappa^2)$ and  the heuristic selection procedure is $\bigO(\kappa^2)$. The number of chunks is bounded from above by $\lceil n/ \kappa \rceil$. This yields a complexity of \mbox{$\bigO((n/ \kappa) (d \kappa^2 +  \kappa^2)) = \bigO(d n \kappa)$}. As data coarsening in each chunk is independent, with $\lceil n/ \kappa \rceil$ parallel processors available the complexity reduces to $\bigO(d \kappa^2)$. 

\subsection{Relation to the MWIS Problem}
\label{sec:relation}
The process of identifying a maximal $\epsilon$-separated set of maximum weight is equivalent to solving the MWIS problem for the weighted graph \mbox{$G^\epsilon = (P,E^\epsilon, w_P)$}. Let $G=(V,E,w)$ be a weighted graph. A set of vertices $S \subseteq V$ is called \emph{independent} in $G$ if no two of its vertices are adjacent or, equivalently, if $S$ is a clique in the complement graph. This corresponds to the separation constraint mentioned earlier, where two vertices are adjacent whenever they are less than a distance of $\epsilon$ apart. The MWIS problem can be expressed as
\begin{equation}
    \tag{P\arabic{prob}}
    \addtocounter{prob}{1}
    \label{prob:mwis}
    \underset{S \subseteq V}{\text{maximize}} \,\,
    \omega(S) \quad
    \text{subject to} \quad
    S \text{ is independent},
\end{equation}
and is \mbox{NP-complete} for a general weighted graph~\cite{karp1972reducibility}, yet, for specific graphs, there exist polynomial-time algorithms~\cite{mandal2006maximum , kohler2016linear}. Note that the QUBO formulation of the MWIS problem in \cite{abbott2018hybrid, hernandez2016novel} is related to the one in \cref{prob:main_problem_QUBO_matrix}.

If all weights are positive, a \emph{maximum} weighted independent set is necessarily a \emph{maximal} independent set. A maximal independent set is a \emph{dominating set}, that is, a subset $S$ of $V$ such that every $v \in V \setminus S$ is adjacent to some $w \in S$. This corresponds to our observation that every maximal \mbox{$\epsilon$-separated} subset is \mbox{$\epsilon$-dense}.

\section{Results}
The datasets used to demonstrate the efficiency and robustness of our approach are the MNIST dataset of handwritten digits~\cite{lecun2010mnist}, a two-dimensional version of MNIST  obtained by using $t$-SNE~\cite{maaten2008visualizing}, two synthetic grid datasets, and a dataset called Covertype containing data on forests in Colorado~\cite{Dua:2017}. The synthetic grid datasets are the unions of 100 samples (in the 2D case) and 1000 samples (in the 3D case) drawn from $\cN(\mu_{ij},\sigma^2)$ with means $\mu_{ij} = (10i+5,10j+5)$ and a variance of $\sigma_2=4$ for $0 \leq i,j \leq 9$ in the 2D case and the natural extension in the 3D case. Dataset statistics are provided in~\cref{tab:tabulars}. In addition to our technical experiments, explained in the following sections, a practical application of BiDViT for image qunatization is illustrated in \cref{fig:application}. All experiments were performed using a 2.5 GHz Intel Core i7 processor and 16 GB of RAM.

\subsection{Low-Range Clustering Domain}
Although BiDViT has been specifically designed for extreme clustering, it  yields accurate assignments for low values of $k$. \cref{fig:low_range} shows the clustering assignment of BiDViT on the 2D grid dataset and on MNIST. The results are obtained by manually selecting a BiDViT level. In the grid dataset, every cluster is identified correctly. In the MNIST dataset, all clusters are recognized, except one. However, as our algorithm is based on metric balls, and some datasets might not conform to such categorization, there are datasets for which it cannot accurately assign clusters. This is true for most clustering algorithms, as they are able to recognize only specific shapes.

\begin{figure}[t]
    \centering
    \includegraphics[clip, trim=1cm 0cm 1cm 1.1cm, width=0.46\textwidth]{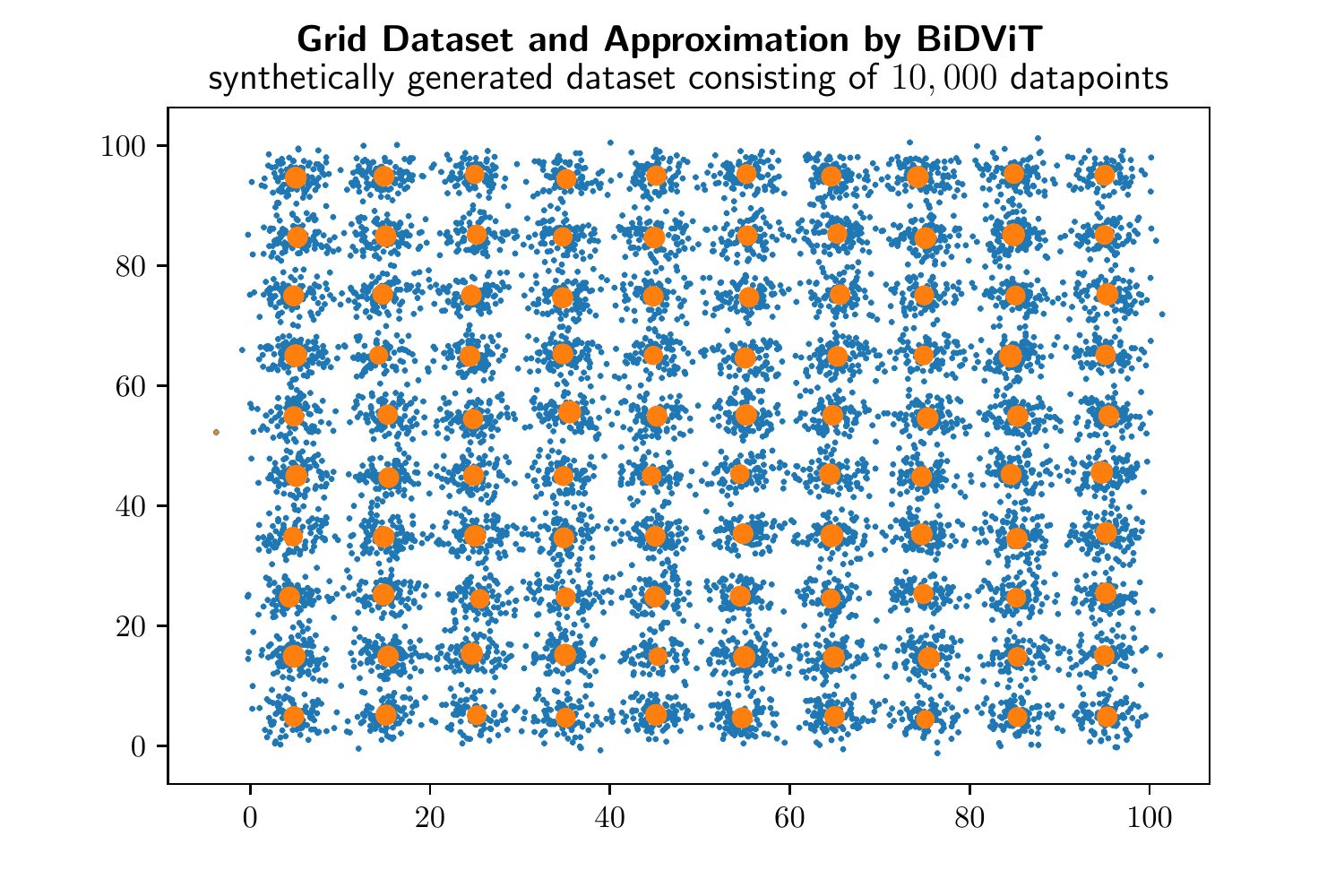}
    \includegraphics[clip, trim=1cm 0cm 1cm 1.1cm, width=0.46\textwidth]{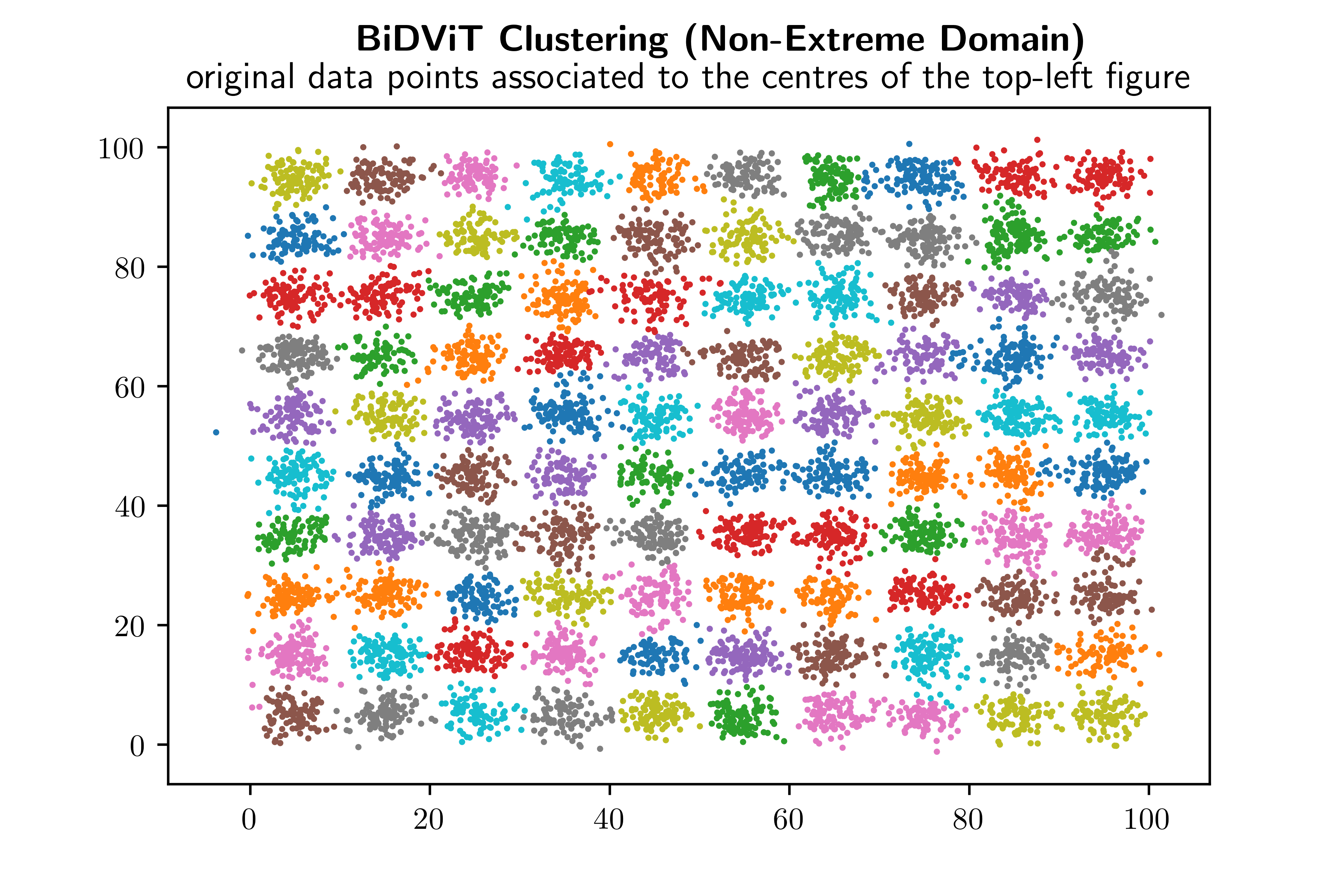}
    \includegraphics[clip, trim=1cm 0cm 1cm 1.1cm, width=0.46\textwidth]{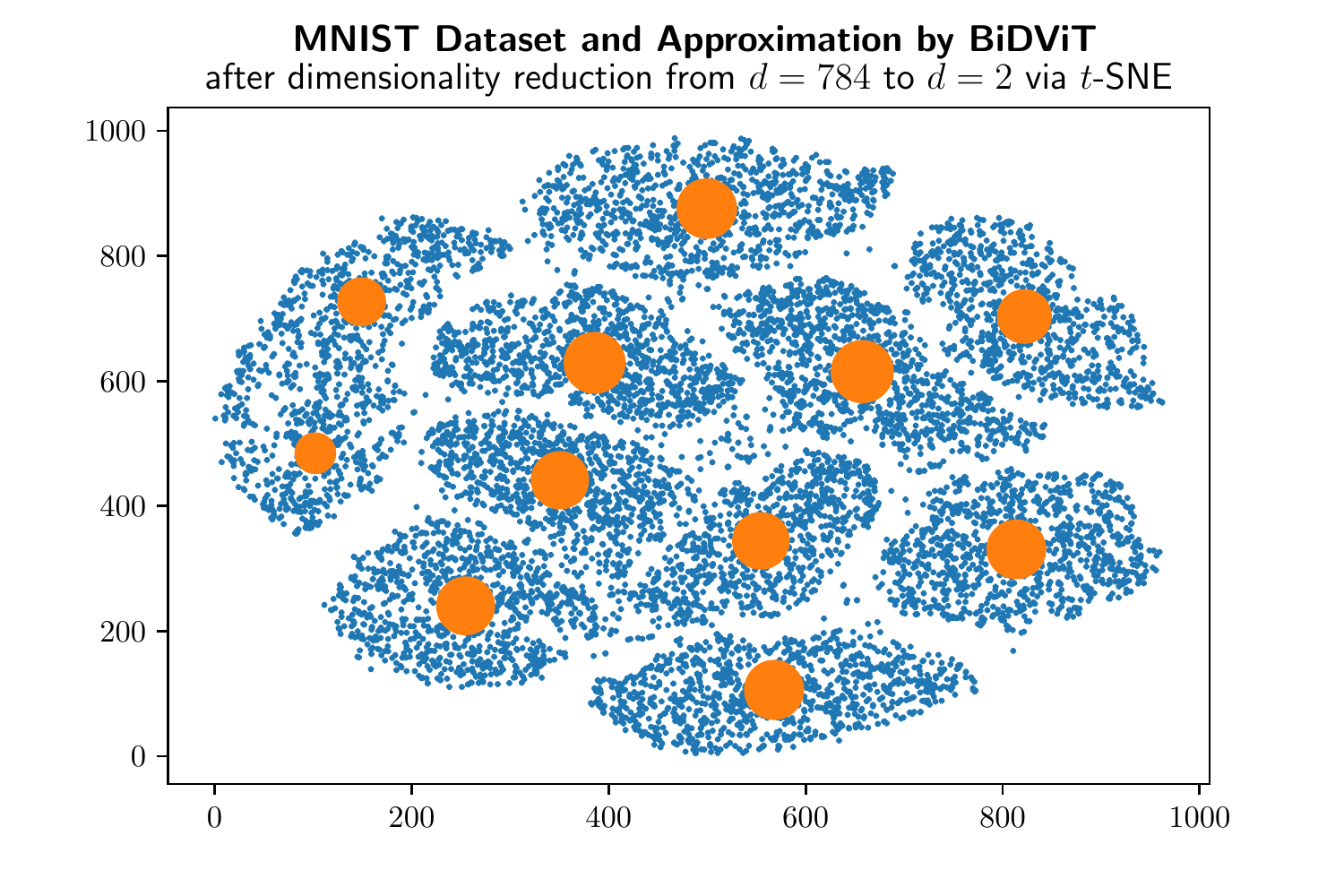}
    \includegraphics[clip, trim=1cm 0cm 1cm 1.1cm, width=0.46\textwidth]{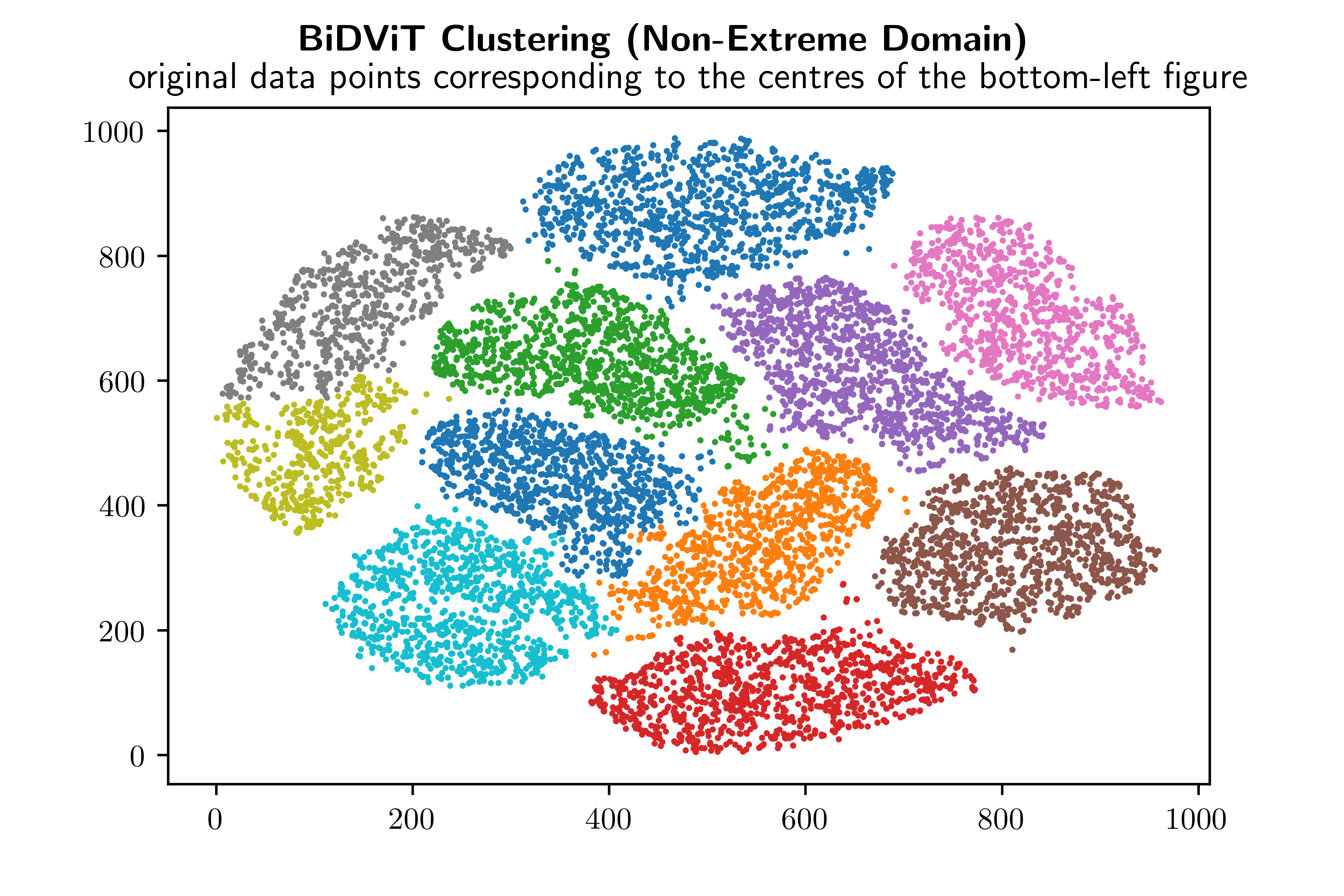}
    \caption{Performance of BiDViT in the non-extreme clustering domain. The left-hand-side figures show the original datasets (blue) with cluster centroids (orange) determined by BiDViT. On the right, colours correspond to assigned labels. The figures can be reproduced by using the parameters $\kappa=10^3, \alpha=1.3,$ and $\epsilon_0=2.0$, and using BiDViT level 18 for the MNIST dataset (bottom), and $\kappa=10^3, \alpha=1.3,$ and $\epsilon_0=1.0$, and BiDViT level 10 for the synthetic grid (top).}
    \label{fig:low_range}
\end{figure}

\subsection{Extreme Clustering Capability}
To evaluate the performance of BiDViT on high-dimensional datasets in the extreme clustering range, we used the \emph{Calinski--Harabasz score}~\cite{calinski1974dendrite} and the \emph{Davies--Bouldin score}~\cite{davies1979cluster}. These clustering metrics are  \emph{internal evaluation schemes}, that is, their values depend solely on the clustered data, not requiring the ground truth label assignment for the dataset. Such schemes must be viewed as heuristic methods: their optimal values do not guarantee optimal clusters but provide a reasonable measure of clustering quality. Detailed analyses have been conducted on the advantages and shortcomings of internal clustering measures~\cite{liu2010understanding, jain2008innovation}. In the extreme clustering scenario, where the objective is to obtain an accurate approximation of the entire dataset instead of categorizing its elements, no \emph{true} labels are given and thus external evaluation schemes (ones based on the distance to a ground truth clustering assignment) do not qualify as success measures.

Let $C_1, \ldots, C_{n_\text{c}}$ denote a total of $n_\text{c}$ detected clusters within a dataset $X$ with $n$ data points. The Calinski--Harabasz score $\mathcal{S}_{\text{CH}}$ of a clustering is defined as a weighted ratio of the \emph{inter-cluster squared deviations} to the sum of the \emph{intra-cluster squared deviations}. More precisely, $\mathcal{S}_{\text{CH}}$ is given by
\begin{equation}
    \mathcal{S}_{\text{CH}}(C_1, \ldots, C_{n_\text{c}}) = \left(  \frac{n-1}{n_\text{c}-1} \right) \frac{\sum_{k=1}^{n_\text{c}} \vert C_k \vert \Vert c_k - c \Vert^2_2}{\sum_{k=1}^{n_\text{c}} \sum_{x \in C_k} \Vert x - c_k \Vert^2_2},
\end{equation}
where $c_k$, for $k=1,\ldots,n_\text{c}$ are the cluster centroids, and $c$ is their mean. High values of $\mathcal{S}_{\text{CH}}$ are indicative of a high clustering quality. The Davies--Bouldin score  $\mathcal{S}_{\text{DB}}$ is the average maximum value of the ratios of the pairwise sums of the intra-cluster deviation to the inter-cluster deviation. The score is defined as
\begin{equation}
    \mathcal{S}_{\text{DB}}(C_1,\ldots, C_{n_\text{c}}) = \frac{1}{n_\text{c}} \sum_{k=1}^{n_\text{c}} \max_{j \neq k} \frac{S_k + S_j}{\Vert c_k - c_j \Vert_2},
\end{equation}
where $S_i = \sum_{x \in C_i} \Vert x - c_i \Vert / \vert C_i \vert.$ Low values of $\mathcal{S}_{\text{DB}}$ indicate accurate clustering.

\cref{fig:scores} shows $\mathcal{S}_{\text{CH}}$ and $\mathcal{S}_{\text{DB}}$ of clustering assignments obtained with BiDViT and Mini Batch $k$-means clustering~\cite{sculley2010web} for different values of $k$ on the Covertype dataset. Due to their high computational complexity with respect to $k$, many common clustering algorithms could not be applied. Remarkably, $\mathcal{S}_{\text{CH}}$ values are quite similar, indicating  that the cluster assignments generated by BiDViT are of comparable quality even though the runtime of our algorithm is significantly shorter. For $\mathcal{S}_{\text{DB}}$, our algorithm outperforms the others for lower values of $k$, and is comparable for large values. One explanation for the slightly weaker performance of BiDViT with respect to $\mathcal{S}_{\text{CH}}$ is that BiDViT aims to minimize the \emph{non-squared} distances, whereas $\mathcal{S}_{\text{CH}}$ rewards clustering methods that minimize \emph{squared} distances. Similarly, this explains BiDViT's advantage for $\mathcal{S}_{\text{DB}}$. 

\begin{figure}[t]
    \centering
    \includegraphics[clip, trim=1.5cm 11.5cm 3cm 2.65cm, width=\textwidth]{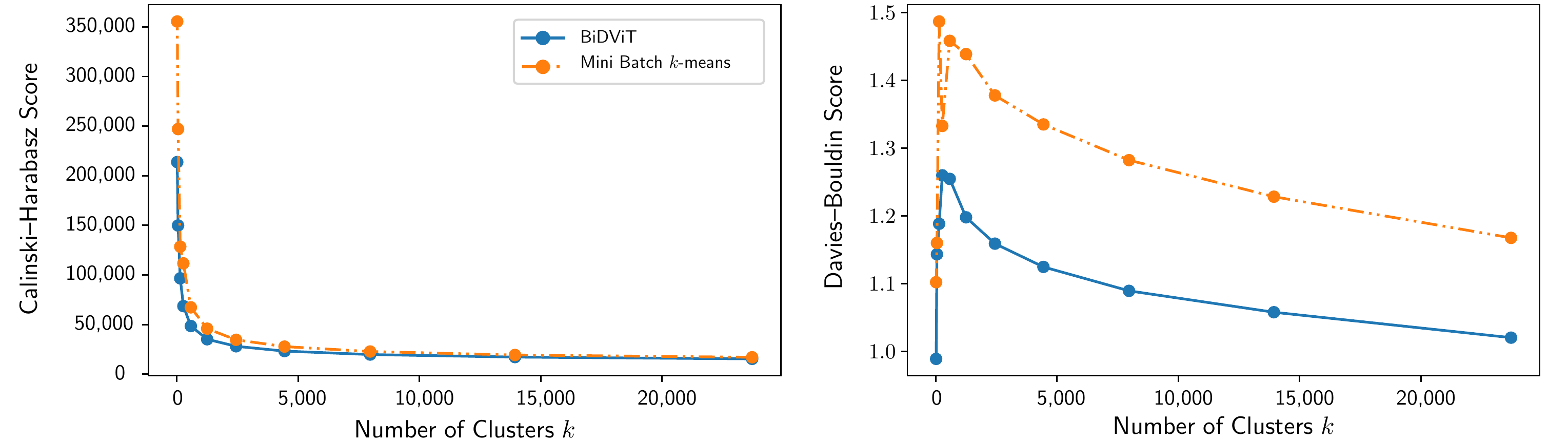}
    \caption{Calinski--Harabasz score  $\mathcal{S}_{\textup{CH}}$ (left) and Davies--Bouldin score $\mathcal{S}_{\textup{DB}}$ (right) of clustering assignments on the Covertype dataset generated by the heuristic BiDViT algorithm ($\kappa=10^3, \alpha=1.5,$ and $\epsilon_0=10^2$) and Mini Batch $k$-means clustering (\texttt{batch\_size} $= 50$, \texttt{max\_iter} $= 10^3$, \texttt{tol} $= 10^{-3},$ and \texttt{n\_init} $= 1$). Whereas a higher value of $\mathcal{S}_{\textup{CH}}$ indicates better clustering, the opposite is the case for $\mathcal{S}_{\textup{DB}}$.}
    \label{fig:scores}
\end{figure}
\begin{figure}[t]
    \centering
    \includegraphics[clip, trim=1cm 22.5cm 2.5cm 4.2cm, width=0.48\textwidth]{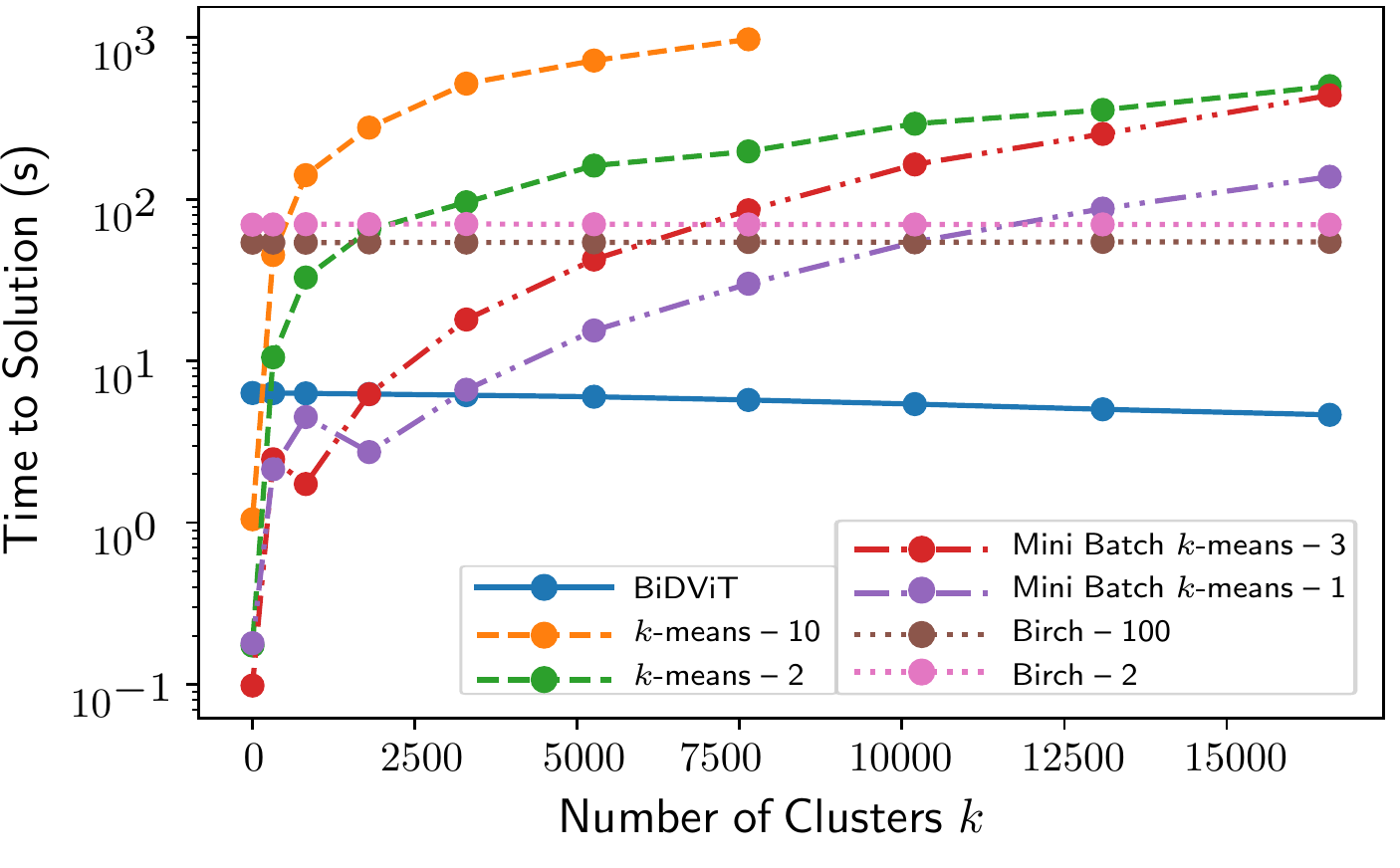}
    \includegraphics[clip, trim=0.3cm 3cm 2.6cm 23.8cm, width=0.51\textwidth]{figures/all_30K.pdf}
    \caption{Time to solution (left) and Davies--Bouldin score (right) of common clustering algorithms and BiDViT ($\kappa=10^3, \alpha=1.3,$ and $\epsilon_0=16.0)$ on a subset of the Covertype dataset for different numbers of clusters. For $k$-means++ and Mini Batch $k$-means clustering, we modified the number of initializations, and for Birch clustering, it was the branching factor. These parameters resulted in a speed-up with a minimum loss of quality; their values are indicated in the legend.}
    \label{fig:runtime_comparison}
\end{figure}

\subsection{Runtime Comparison}
In our experiments, we observed that, with respect to the total runtime, even the heuristic version of BiDViT restricted to a single core outperforms common clustering methods in the extreme clustering domain. \cref{fig:runtime_comparison} shows the runtime required by different clustering algorithms for the Covertype dataset. For the implementation of methods other than BiDViT, we used the publicly available \texttt{sklearn.clustering} module for Python. To generate the plots, we ran the entire BiDViT procedure, then applied classical algorithms for the same values of $k$. The results suggest that, in the extreme clustering domain, the runtime of BiDViT is an order of magnitude faster than that of the agglomerative methods against which it was compared, and multiple orders of magnitude faster than that of $k$-means and Mini Batch \mbox{$k$-means} clustering. The dataset cardinality was restricted to 20,000 points to obtain results for other methods, whereas BiDViT is capable of handling the entire dataset comprising 581,000 points.

\begin{table}[b!]
\centering
\caption{Dataset statistics (left) and runtime comparison of extreme clustering algorithms in seconds (right). \emph{PERCH-C} (``collapsed-mode'') was run, as it outperforms standard PERCH. The parameter $L$ sets the maximum number of leaves (see~\cite{kobren2017hierarchical} for an explanation). BiDViT selected the values \mbox{$\epsilon_0=30$} and \mbox{$\epsilon_0=0.5$}, such that a  percentage of the nodes collapsed in the initial iteration,  for the Covertype and the MNIST datasets, respectively. The mean and standard deviation were computed over five runs.}
\resizebox{\textwidth}{!}{
\renewcommand{\arraystretch}{2}
\begin{tabular}{cccc}
\toprule
\bf Name & \bf Description & \bf Cardinality  & \bf Dimension \\
\midrule
MNIST & handwritten images & $60$ K & $784$ \\
MNIST-2D & $t$-SNE of the above & $60$ K & $2$ \\
Covertype & forest data & $581$ K & $54$ \\
grid-2D & synthetically generated & $100$ K & $2$ \\
grid-3D & synthetically generated & $1000$ K & $3$ \\
\bottomrule
\end{tabular}
\quad
\renewcommand{\arraystretch}{1}
\begin{tabular}{ccc}
\toprule
\bf Algorithm & \multicolumn{2}{c}{\bf Runtime on Dataset (seconds)} \\
specified parameters &  Covertype & grid-3D  \\
\midrule
\multirow{1}{*}{PERCH-C}& \multirow{2}{*}{$1616.45 \pm 20.37$} & \multirow{2}{*}{$1588.10 \pm 41.46$}  \\ 
\scriptsize{$L=\textup{Inf}$} & & \\
\multirow{1}{*}{PERCH-C}& \multirow{2}{*}{$1232.53 \pm 53.61$} & \multirow{2}{*}{$1280.30 \pm 15.03$}  \\ 
\scriptsize{$L=50,000$} & & \\
\multirow{1}{*}{PERCH-C}& \multirow{2}{*}{$928.82 \pm 47.00$} & \multirow{2}{*}{--}  \\ 
\scriptsize{$L=10,000$} & & \\
\multirow{1}{*}{BiDViT (heuristic)} &  \multirow{2}{*}{$301.36 \pm 10.01$} & \multirow{2}{*}{$152.50 \pm 0.86$ }  \\
\scriptsize{$\kappa=2000, \alpha=1.1$}& &  \\
\multirow{1}{*}{BiDViT (heuristic)} &  \multirow{2}{*}{$56.26 \pm 0.62$} & \multirow{2}{*}{$75.22 \pm 0.95$}  \\
\scriptsize{$\kappa=500, \alpha=1.2$}& & \\
\bottomrule
\end{tabular}}
\label{tab:tabulars}
\end{table}

We then compared the runtime of BiDViT to \emph{PERCH} (``Purity Enhancing Rotations for Cluster Hierarchies''), a hierarchical algorithm for extreme clustering~\cite{kobren2017hierarchical}, to our knowledge the only other algorithm designed to solve extreme clustering problems. We restricted both algorithms to using a single core. \cref{tab:tabulars} shows that BiDViT performs an order of magnitude faster than PERCH. However, they solve somewhat different problems: whereas BiDViT aims to gradually coarsen a dataset by finding $\epsilon$-separated, \mbox{$\epsilon$-dense} subsets, PERCH maximizes the \emph{dendrogram purity}, a measure of the clustering tree's consistency~\cite{kobren2017hierarchical}. The clustering tree generated by PERCH is binary and thus enormous, allowing for much finer incremental distinctions between clustering assignments. In contrast, the tree generated by BiDViT is more compact, as multiple data points can collapse into the same representative point. When comparing dendrogram purities, we expect PERCH to outperform BiDViT; when comparing Davies--Bouldin scores at a given level, we expect the opposite. We did not test these hypotheses, as dendrogram purity is an external evaluation scheme, that is, it requires a clustering assignment to use for comparison, which is not available in unsupervised machine learning.

\subsection{Results for the Quantum Version of BiDViT}
We tested a  prototype of BiDViT on a \mbox{D-Wave 2000Q} quantum annealer, a machine that has 2048 qubits and 5600 couplers. According to \mbox{D-Wave Systems}, the computer uses 128,000 Josephson junctions and was the most complex superconducting integrated circuit built to date when introduced in January of 2017~\cite{DWave}.

Before solving the QUBO problems, we applied preprocessing techniques, reducing their size and difficulty \cite{glover2018logical}. This proved effective and eliminated a great many variables. In most cases, we observed a size reduction of over 60\%.  

For the quantum version of BiDViT, we observed higher-quality solutions and a significant speed-up for BiDViT, when compared to common clustering methods. Both observations are based on results shown in \cref{fig:D-Wave_results}.

\begin{figure}[t]
    \centering
    \includegraphics[clip, trim=0cm 0cm 0cm 0cm, width=.5\textwidth]{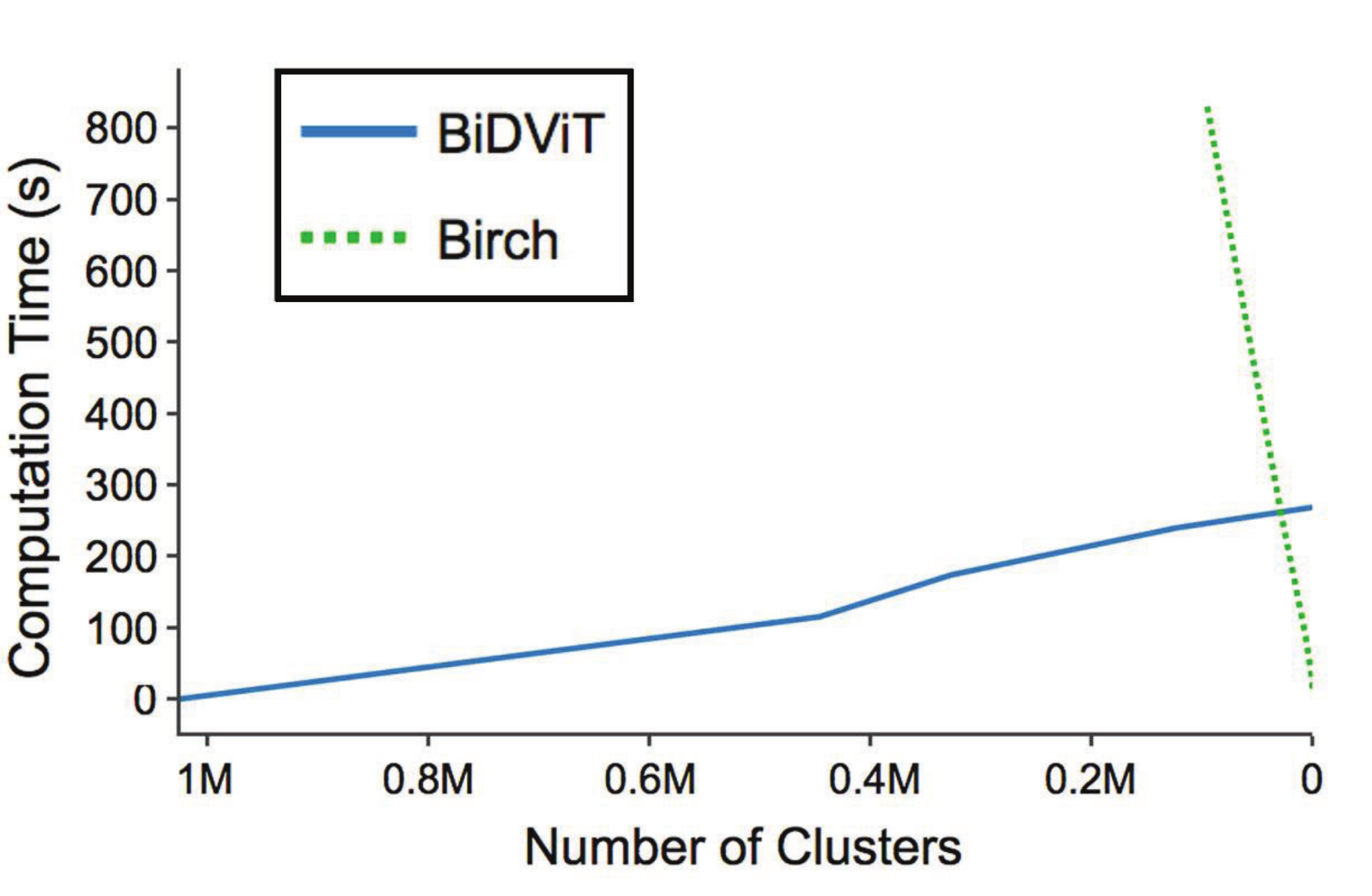}
    \includegraphics[width=0.48\textwidth]{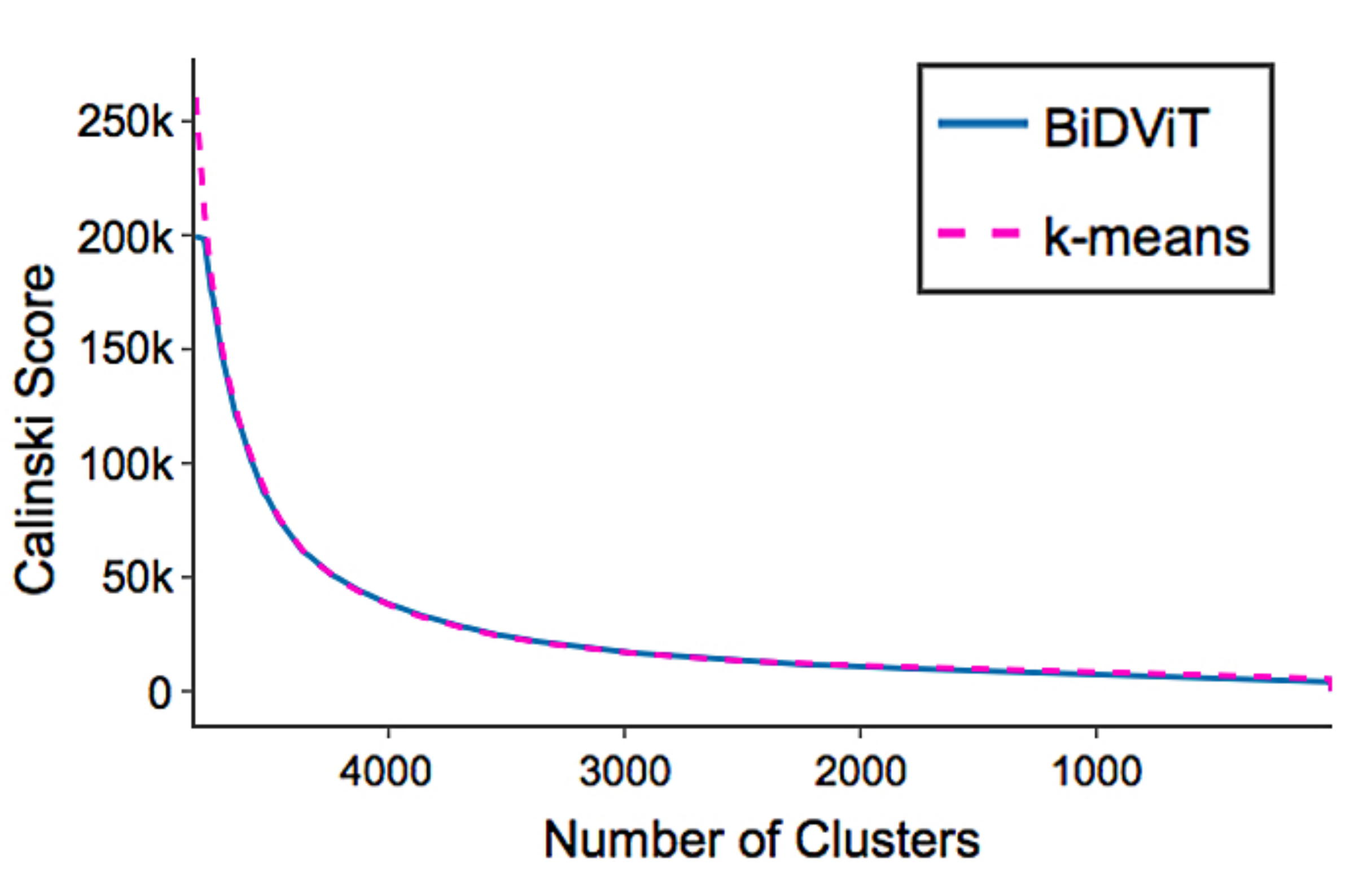}
    \caption{Runtime and quality of  results for the quantum version of BiDViT obtained using a D-Wave 2000Q quantum annealer. Left) Computational time for the 3D grid dataset. Right) Comparison of the Calinski--Harabasz score of the quantum version of BiDViT and of $k$-means clustering on a subset of the MNIST dataset for different numbers of clusters. We chose to invert the orientation of the abscissae 
     to illustrate that at low BiDViT levels there are many clusters and at high levels only a few remain.}
    \label{fig:D-Wave_results}
\end{figure}

However, the heuristic version of BiDViT and the common clustering algorithms were executed on a classical device that has a limited computational capacity, whereas the \mbox{D-Wave 2000Q} is a highly specialized device. Running these algorithms on a high-performance computer might lead to an equivalent degree of speed-up.

\section{Conclusion}
We have developed an efficient algorithm capable of performing extreme clustering. Our complexity analysis and numerical experiments show that if the dataset cardinality and the desired number of clusters are both large, the runtime of \mbox{BiDViT} is at least an order of magnitude faster than that of classical algorithms, while yielding a solution of comparable quality. With advances in quantum annealing hardware, one can expect further speed-ups in our algorithm and size of dataset that can be processed. 

Independent of BiDViT, our coarsening method, based on identifying an $\epsilon$-dense, $\epsilon$-separated subset, is valuable in its own right---it is a novel approach to clustering  which is not limited solely to the extreme clustering domain.

Further investigation of our coarsening approach is justified, as we have identified a domain for the radius of interest (in \cref{thm:separability}) such that, under a separability assumption, every solution to \cref{prob:original_problem} (i.e., every maximum weighted $\epsilon$-separated subset) yields the optimal clustering assignment.

\begin{figure}[h!]
    \centering
    \includegraphics[clip, trim=2cm 2cm 2cm 2cm, width=0.64\textwidth]{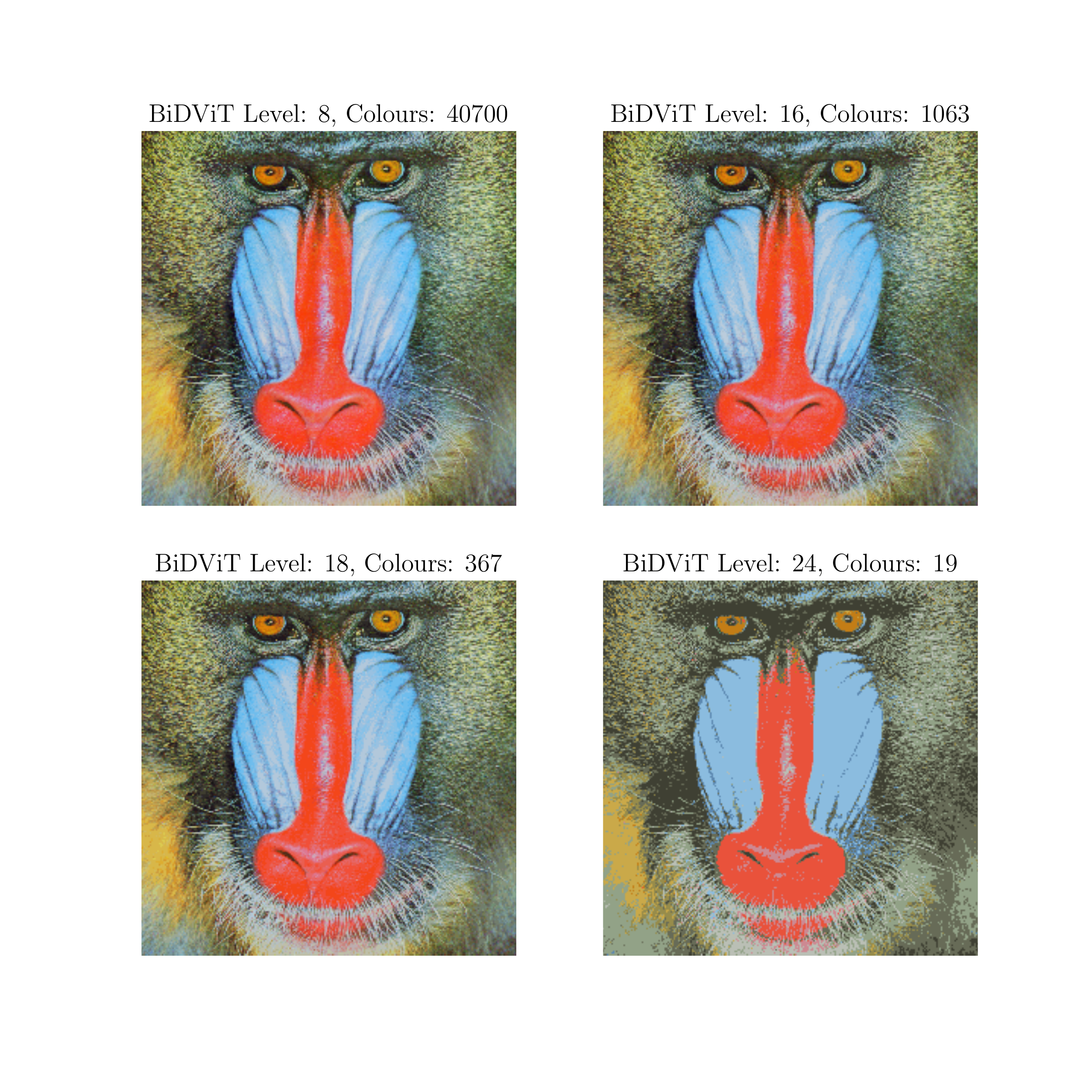}
    \includegraphics[clip, trim=1cm 2cm 0cm 2cm, width=0.35\textwidth]{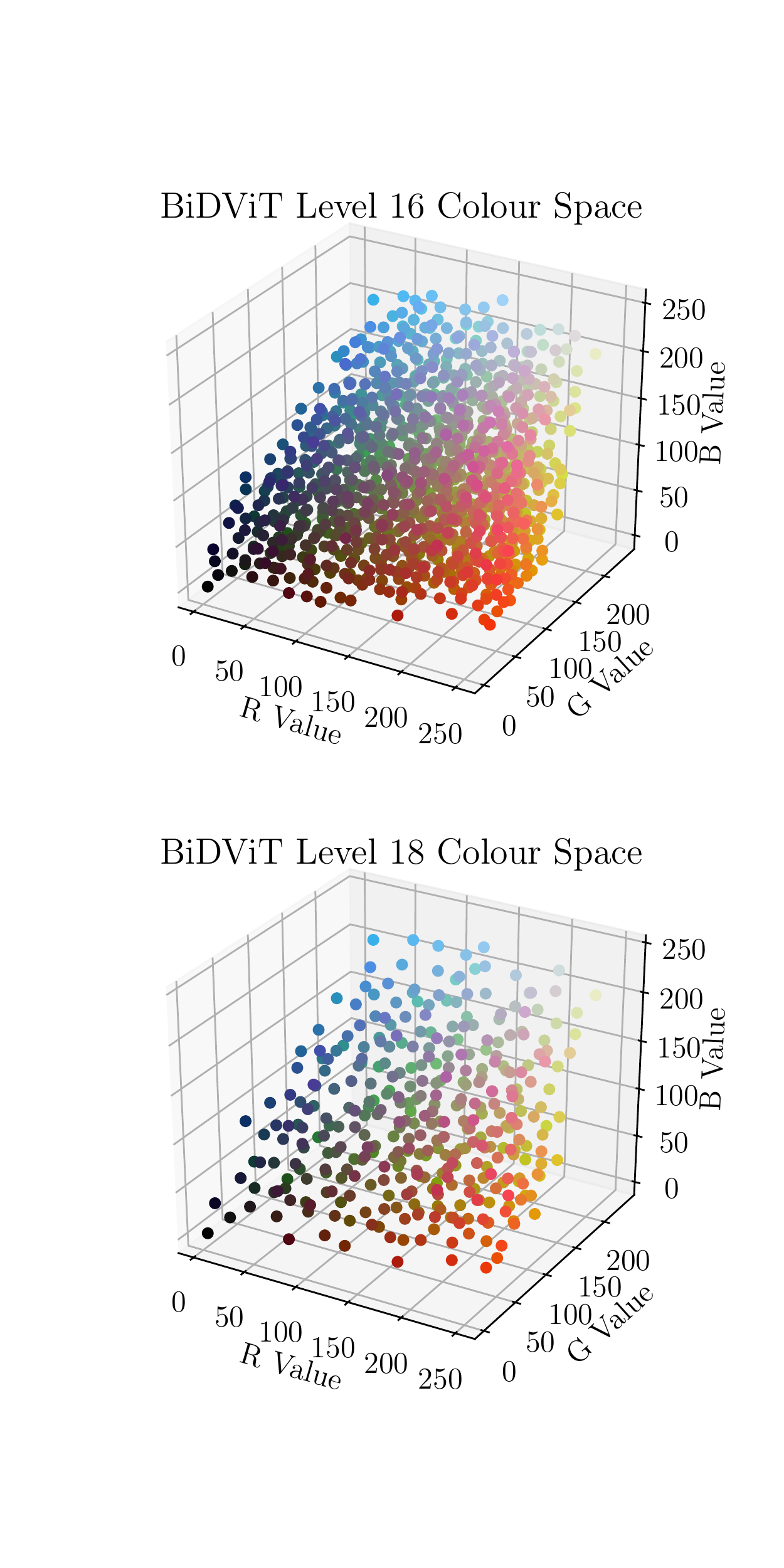}
    \caption{Image quantization via clustering in the colour space of a standard test image. The original image has 230,427 colours. BiDViT is particularly fast at reducing its colours to a number on the order of $10^4$, as this falls into the extreme clustering range. Here, the $k$-means clustering algorithm faces its computational bottleneck. A commonly employed algorithm for such problems is the median cut algorithm. Naturally, it is faster than BiDViT---as BiDViT employs the median cut algorithm in its chunking procedure---but BiDViT produces a more accurate colour assignment.}
    \label{fig:application}
\end{figure}

\section*{Acknowledgements}
We thank Saeid Allahdadian, Nick Cond\'e, Daniel Crawford, and Austin Wallace for contributing to an earlier version of the algorithm. We thank Maliheh Aramon, Pooja Pandey, and Brad Woods for helpful discussions on optimization theory. The implementation of the QUBO preprocessing techniques was performed jointly with Brad Woods and Nick Cond\'e. Inderpreet Singh contributed to the figure on image quantization. Victoria Wong assisted with graphical editing of two figures. Partial funding for this work was provided by the Mitacs Accelarate internship initiative.

\bibliographystyle{bibtex/splncs03}
\bibliography{library}

\newpage
\appendix

\section*{Supplementary Information to the Paper, ``A Quantum Annealing-Based Approach to Extreme Clustering''}

\subsection*{Appendix A: An Alternative Coarsening Method}
\label{sec:alt_approx}

In certain situations, a user might not want the approximating set to be $\epsilon$-separated but instead might be interested in finding an $\epsilon$-dense subset with a minimum number of elements, or, more generally, the minimum cost for some \emph{cost function} $c: X \rightarrow \mathbb{N}$.  Finding such a set can be realized in a very similar way to the quantum method of BiDViT. In fact, the only modifications needed would be \mbox{to \cref{sec:chunk_approx}} in the paper, where we introduce the concept of chunk coarsening. 

Let $P = \lbrace x^{(1)}, \ldots, x^{(n)} \rbrace$, and let $N^{(\epsilon)}$ and $s_i$, for $i=1,\ldots,n$, be defined as \mbox{in \cref{sec:chunk_approx}}. Analogously to the weight vector $w$, we define a cost vector $c$ by $c_i= c(x^{(i)})$ for each $x^{(i)} \in P$. The problem of finding an $\epsilon$-dense subset $S \subseteq P$ of minimum cost can then be expressed as follows:
\begin{equation}
    \tag{P\arabic{prob}}
    \addtocounter{prob}{1}
    \label{eq:MSC}
    \underset{s \in \lbrace 0,1 \rbrace^n}{\text{minimize}}\,\,
    \sum_{i=1}^n s_i c_i \quad
    \text{subject to} \quad
    \sum_{j=1}^n N_{ij}^{(\epsilon)} s_j \geq 1, \quad i=1,\ldots n. 
\end{equation}
The constraints in \cref{eq:MSC} enforce the condition that for each solution (corresponding to a subset), every point in $P$ is represented by at least one of the points from the selected subset. The subset will not necessarily be $\epsilon$-separated, but it will be \mbox{$\epsilon$-dense}.

In the same way that finding an $\epsilon$-separated subset of maximum weight corresponds to the MWIS problem, finding an $\epsilon$-dense subset of minimum cost corresponds to the minimum weighted dominating set (MWDS) problem, which is equivalent to a weighted version of the minimal set covering (MSC) problem. Consider a set $U$ and subsets \mbox{$S_j \subseteq U$ and $j \in J$}, where $J$ is some set of indices, such that $U = \bigcup_{j \in J} S_j$. The MSC problem then consists of finding a subset $J_0 \subseteq J$ such that the property $U \subseteq \bigcup_{j \in J_0} S_j$ is satisfied, and $J_0$ is of minimum cardinality with respect to this property. For example, if \mbox{$U = \{a,b,c,d,e\}$}, \mbox{$S_1 = \{a,c\}$}, $S_2 = \{a,d\}$, and $S_3 = \{b,d,e\}$, then the solution to the MSC problem is given by $J_0= \{1,3\}$, as none of the subsets cover $U$, but the union $S_1 \cup S_3$ does. The general MSC problem is known to be NP-hard \cite{karp1972reducibility}. By defining $S_j = B(x^{(j)},\epsilon) \cap P$ for $j=1,\ldots, n$ in the above setting, one can see that we have solved a weighted version of the MSC problem.  

To transform \cref{eq:MSC} into a QUBO problem, we convert the inequality constraints to equality constraints by adding integer slack variables. Note that the $i$-th constraint is satisfied if and only if there exists some $\xi_i \in \mathbb{N}_0$ such that
$ \sum_{j=1}^n N_{ij}^{(\epsilon)} s_j - 1 = \xi_i$.
In fact, given that $s \in \lbrace 0, 1 \rbrace^n$, we can see that the $\xi_i$ must satisfy the bounds
\begin{equation}
    \label{eq:bounds}
    0 \leq\xi_i \leq \left(\sum_{j=1}^n N_{ij}^{(\epsilon)}\right) -1, \quad \textup{ for } i=1,\ldots,n.
\end{equation}
Thus, by dualizing the equality constraints, \cref{eq:MSC} can be expressed as a QUBO problem
\begin{equation}
    \tag{P\arabic{prob}}
    \addtocounter{prob}{1}
    \label{eq:INT}
    \underset{\substack{s \in \lbrace 0,1 \rbrace^n\\ \mathbf{0}\leq \xi \leq ( N^{(\epsilon)}\mathbf{1}) - \mathbf{1}}}{\text{minimize}}\,\,
    \sum_{i=1}^n s_i c_i + \lambda \sum_{i=1}^n \left[ \left( \sum_{j=1}^n N_{ij}^{(\epsilon)} s_j \right) - 1 - \xi_i \right]^2.
\end{equation}

We will now describe how substituting a binary encoding for each of the $\xi_i$, for $i=1, \ldots, n$, in \cref{eq:INT} yields the desired QUBO formulation. For each $i=1,\ldots, n$, the $(N^{(\epsilon)} \mathbf{1})_i$ possible states of $\xi_i$ can be encoded by $\lfloor m_i \rfloor +1$ binary variables $b^{(i)}_0, \ldots, b^{(i)}_{\lfloor m_i \rfloor}$, where \mbox{$m_i = \log_2(N^{(\epsilon)} \mathbf{1})_i $}. The encoding has the form

\begin{equation}
    \xi_i = \sum_{k=0}^{\lfloor m_i \rfloor} b^{(i)}_k \gamma_k^{(i)}, \quad \text{for } i=1,\ldots,n,
\end{equation}
where $\gamma_k^{(i)} \in \mathbb{N}$ are fixed coefficients that depend solely on the bounds of \cref{eq:bounds}. If we were to select $\gamma_k^{(i)} = 2^k$ for \mbox{$k=0, \ldots, \lfloor m_i \rfloor$}, then, if $m_i \notin \mathbb{N}$, $\xi_i$ could assume states that do not satisfy these bounds. We can avoid this situation by manipulating the coefficient $\gamma^{(i)}_{\lfloor m_i \rfloor}$ of the final bit $b_{\lfloor m_i \rfloor}^{(i)}$ such that $\sum_{k=0}^{\lfloor m_i \rfloor - 1}  2^k + \gamma^{(i)}_{\lfloor m_i \rfloor}  = (N^{(\epsilon)} \mathbf{1})_i -1$. This may lead to a situation where there are multiple valid encodings for the same integer, but it will always hold that
\begin{equation}
    0 \leq \sum_{k=0}^{\lfloor m_i \rfloor}  b^{(i)}_k \gamma^{(i)}_{k} \leq \left(\sum_{j=1}^n N_{ij}^{(\epsilon)}\right) -1,
\end{equation}
where $\gamma^{(i)}_k = 2^k$ for $k< \lfloor m_i \rfloor$. Substituting the binary encoding into \cref{eq:INT} yields the following QUBO formulation:
\begin{equation}
    \tag{P\arabic{prob}}
    \addtocounter{prob}{1}
    \label{eq:QUBO_MSC}
    \underset{\substack{b^{(i)} \in \lbrace 0,1 \rbrace^{\lfloor m_i \rfloor + 1}\\ s \in \lbrace 0,1 \rbrace^n}}{\text{minimize}}\,\,
    \hspace{-0.4em} \sum_{i=1}^n s_i c_i + \lambda \left[ \left(\sum_{j=1}^n N_{ij}^{(\epsilon)} s_j \right) -1 - \sum_{k=0}^{\lfloor m_i \rfloor}  b^{(i)}_k \gamma^{(i)}_{k} \right]^2. 
\end{equation}

One can show that the solution set of this QUBO problem is equivalent to the one for \cref{eq:MSC} for $\lambda > n \Vert c \Vert_\infty$. We have not investigated whether this bound is sharp. Note that our QUBO formulation is similar to the one described in~\cite{lucas2014ising}, but uses a different encoding.

The number of binary variables in the QUBO formulation of this problem depends on the binary encoding of $\xi$. If the vertex degree in $G^\epsilon$ is uniformly bounded from above by a constant $\eta>0$, then each $\xi_i$ can be encoded with fewer than $\lfloor \log_2(\eta) \rfloor$ binary variables. Therefore, the number of variables in the QUBO polynomial will be at most \mbox{$n(1+\lfloor \log_2(\eta) \rfloor)$}. In the worst case, that is, when there is vertex that is a neighbour of every other vertex, the polynomial would still comprise fewer than $n(1 + \lfloor \log_2(n) \rfloor)$ variables.
\end{document}